\documentclass{article}

\usepackage[preprint, nonatbib]{neurips_2024}

\usepackage[utf8]{inputenc} 
\usepackage[T1]{fontenc}    
\usepackage{hyperref}       
\usepackage{url}            
\usepackage{booktabs}       
\usepackage{amsfonts}       
\usepackage{nicefrac}       
\usepackage{microtype}      
\usepackage{xcolor}         

\usepackage{graphicx}
\usepackage{subcaption}

\usepackage{amsmath}
\usepackage{amssymb}
\usepackage{mathtools}
\usepackage{amsthm}

\usepackage{algorithmic}
\usepackage{algorithm}

\usepackage{multirow}
\usepackage{bm}
\newtheorem{lemma}{Lemma}

\usepackage{wrapfig}

\title{Invisible Backdoor Attacks on Diffusion Models}

\author{%
   Sen Li$^1$,\hspace{0.5em} Junchi Ma$^2$,\hspace{0.5em} Minhao Cheng$^3$ \\
   $^1$The Hong Kong University of Science and Technology, Hong Kong, China \\
   $^2$The University of British Columbia, Canada \\
   $^3$Penn State University, USA \\
   \texttt{slien@connect.ust.hk}, \texttt{junchima@student.ubc.ca}, \texttt{mmc7149@psu.edu}
}

\begin{document}
\maketitle

\begin{abstract}
  In recent years, diffusion models have achieved remarkable success in the realm of high-quality image generation, garnering increased attention. This surge in interest is paralleled by a growing concern over the security threats associated with diffusion models, largely attributed to their susceptibility to malicious exploitation. Notably, recent research has brought to light the vulnerability of diffusion models to backdoor attacks, enabling the generation of specific target images through corresponding triggers. However, prevailing backdoor attack methods rely on manually crafted trigger generation functions, often manifesting as discernible patterns incorporated into input noise, thus rendering them susceptible to human detection. In this paper, we present an innovative and versatile optimization framework designed to acquire invisible triggers, enhancing the stealthiness and resilience of inserted backdoors. Our proposed framework is applicable to both unconditional and conditional diffusion models, and notably, we are the pioneers in demonstrating the backdooring of diffusion models within the context of text-guided image editing and inpainting pipelines. Moreover, we also show that the backdoors in the conditional generation can be directly applied to model watermarking for model ownership verification, which further boosts the significance of the proposed framework.
  Extensive experiments on various commonly used samplers and datasets verify the efficacy and stealthiness of the proposed framework. Our code is publicly available at \url{https://github.com/invisibleTriggerDiffusion/invisible_triggers_for_diffusion}.
\end{abstract}

\section{Introduction}
\label{sec:intro}
\vspace{-0.7em}
Recently, diffusion models have showcased exceptional performance in generating high-quality and diverse image~\cite{ho2020denoising, nichol2021improved, dhariwal2021diffusion, ho2022classifier, nichol2021glide}. Based on diffusion models, many popular applications have been developed including GLIDE~\cite{nichol2021glide}, Imagen~\cite{saharia2022photorealistic}, and Stable Diffusion~\cite{rombach2022high}. These applications serve as powerful tools for unleashing creativity in content generation.
By slowly adding random noise to data, diffusion models learn to reverse the diffusion process to construct desired data samples from the noise.  While this approach has proven excellent in creative content generation, it concurrently introduces novel security challenges that warrant careful consideration.

Instances of backdoor attacks on diffusion models, as explored in previous studies~\cite{chou2023backdoor, chen2023trojdiff, chou2023villandiffusion, struppek2022rickrolling}, have illuminated the potential threats associated with manipulating the diffusion process. However, prevailing approaches either incorporate conspicuous image triggers and additional text into the input noise or prompt to backdoor diffusion models~\cite{chou2023backdoor, chen2023trojdiff, chou2023villandiffusion}. Alternatively, some methods focus solely on substituting input text characters with non-Latin characters, impacting text encoders rather than diffusion models~\cite{struppek2022rickrolling}.
While these strategies have demonstrated commendable success rates, the visibility of triggers in~\cite{chou2023backdoor} and~\cite{chen2023trojdiff} renders them susceptible to detection through human inspection. Notably, none of the prior works, to the best of our knowledge, have ventured into the realm of invisible image triggers for backdooring diffusion models. The adaptability of invisible triggers to different input conditions allows them to seamlessly blend into the background noise, enhancing their robustness against diverse inputs.
Moreover, the subtlety of invisible triggers contributes to a more sustained and persistent backdoor presence. This subtlety facilitates prolonged content manipulation, presenting a challenge for defensive mechanisms to promptly identify and effectively mitigate the threat. The nuanced nature of invisible triggers, therefore, adds an extra layer of complexity and resilience to backdoor attacks on diffusion models.

\begin{figure}[htbp]
    \centering
    \includegraphics[width=1\linewidth]{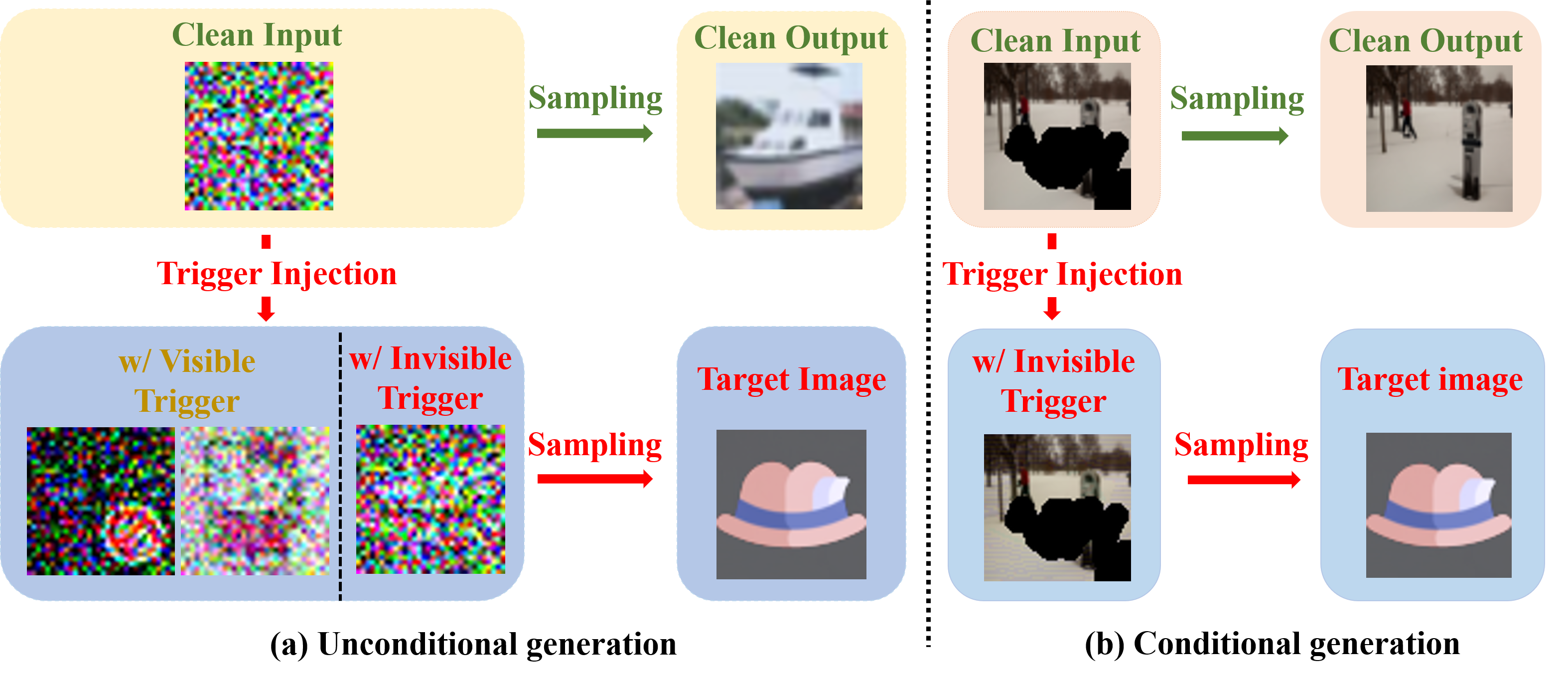}
      \vspace{-1em}
    \caption{Illustration of our proposed invisible triggers, and visible triggers used in~\cite{chou2023backdoor, chou2023villandiffusion, chen2023trojdiff}. }
  
    \label{fig:trigger}
        \vspace{-1em}
\end{figure}

In this paper, illustrated in Figure~\ref{fig:trigger}, we delve into the realm of backdoor attacks featuring invisible image triggers applicable to both unconditional and conditional diffusion models. Our approach introduces a pioneering and comprehensive framework, formulated through bi-level optimization, designed to learn input-aware invisible triggers. The inner optimization phase involves optimizing a trigger generator, given a fixed diffusion model, to enable the generation of target images while maintaining the imperceptibility of the trigger. Simultaneously, the outer optimization phase, given a fixed trigger generator, optimizes the diffusion model to exhibit strong performance on both clean and poisoned data. This framework seamlessly integrates backdooring for both conditional and unconditional diffusion models within a unified optimization problem.

Specifically, when the prior is a random Gaussian noise for the unconditional diffusion model, we train the diffusion model to learn different distribution mappings so that the diffusion model would generate the target image given an initial noise that contains the trigger. 
Importantly, our framework can insert multiple trigger-target pairs and generate instance-adaptive initial noise rather than applying a uniform perturbation. 
For conditional diffusion models, we employ a simple neural network to generate instance-adaptive and invisible triggers, effectively incorporating backdoors into additional priors. This enables the conditional diffusion model to generate the target image regardless of the given text, representing a novel contribution compared to previous works focusing on trigger design in the text domain.
The invisibility of our proposed backdoor trigger also provides us a seamlessly model watermarking method for model ownership verification. If the suspected model is derived from the watermarked model, once the trigger is being activated, the model would generate a designated target image regardless of any prompt/instruction given. 
Unlike the previous watermark method using prompt word as trigger, our proposed watermark is robust against different kinds of image editing and instructions.

In summary, our work marks a significant milestone as we introduce a novel and versatile optimization framework to inject input-aware invisible image triggers into both unconditional and conditional diffusion models. This enhancement renders the injected backdoor more covert and robust. 
\vspace{-1em}

\section{Related work}
\label{sec:related}
\vspace{-0.7em}

\paragraph{Diffusion models} As a new family of powerful generative models, diffusion models could achieve superb performance on high-quality image synthesis~\cite{ho2020denoising, nichol2021glide, saharia2022photorealistic, rombach2022high}. They have shown impressive results on various tasks, such as class-to-image generation~\cite{dhariwal2021diffusion, ho2022classifier}, text-to-image generation~\cite{saharia2022photorealistic}, image-to-image translation~\cite{meng2021sdedit}, text-guided image editing/inpainting~\cite{nichol2021glide, rombach2022high}, and so on. A more detailed introduction to diffusion models can be found in Appendix~\ref{sec:diffusion}. With the powerful capability, the research community has started to focus on the potential security issues that diffusion models may introduce. In this paper, we propose a strong attack framework which can make diffusion models perform maliciously when some invisible patterns are injected into the input, revealing the potential severe security risk that previous works did not cover.

\vspace{-1em}
\paragraph{Backdoor attacks on diffusion models} Recently, diffusion models have been shown to be vulnerable to backdoor attacks~\cite{chou2023backdoor, chen2023trojdiff, chou2023villandiffusion, struppek2022rickrolling}. Struppek et al.~\cite{struppek2022rickrolling} proposed to backdoor the text encoder only in text-to-image diffusion models by replacing text characters with non-Latin characters, showing the potential threat in text-to-image generation. However, their method did not consider the backdoor threat on the diffusion models, limiting the practical use of the method. Chou et al.~\cite{chou2023backdoor} and Chen et al.~\cite{chen2023trojdiff} proposed to backdoor diffusion models in different ways. Their methods focus on backdooring unconditional diffusion models, which may not be applicable to backdoor diffusion models in conditional case. Very recently, Chou et al.~\cite{chou2023villandiffusion} proposed a unified framework on backdoor attack for both unconditional diffusion models and text-to-image diffusion models. 

To the best of our knowledge, all previous works are working on visible trigger in the diffusion model which could be easily detected and ruled out. In this paper, we propose a novel and general framework to inject invisible triggers into both unconditional and conditional diffusion models.
Although there are previous works utilizing bi-level optimization to generate invisible triggers in classification models~\cite{doan2021lira, doan2021backdoor}, the context of learning an invisible backdoor trigger in diffusion models significantly differs. We have included an in-depth discussion in Appendix~\ref{sec:invisible_cls}. Our work further distinguishes itself by formulating a general loss function capable of accommodating various efficient samplers, such as DDIM~\cite{song2020denoising} and DPMSolver~\cite{lu2022dpm}, as opposed to being limited to a single sampler as in~\cite{chou2023backdoor}.

\vspace{-1em}
\section{Methodology}
\label{sec:method}
\vspace{-0.7em}

\subsection{Threat Model}
\vspace{-0.7em}
We adopt a similar threat model as prior research~\cite{chou2023backdoor, chou2023villandiffusion}, where the attack involves two distinct parties. An \textit{attacker} is responsible for injecting backdoors into diffusion models, subsequently releasing the backdoored models. Meanwhile, a \textit{user} downloads these pre-trained models from the web for practical use and have full access to the backdoored models. Additionally, the \textit{user} possesses a subset of clean data for evaluating the models' performance.
Within this model, the \textit{attacker} retains control over the training procedure of the diffusion, encompassing both initial training and fine-tuning. The \textit{attacker} is also granted the capability to modify the training datasets, allowing the addition of supplementary examples. Throughout the training process, the \textit{attacker} endeavors to release backdoored models that exhibit designated behavior (e.g. generating specific image) when the input is injected with the trigger, while behaving normally on inputs without the trigger. Therefore, the \textit{attacker}'s dual objectives include achieving high specificity, ensuring the backdoored models perform maliciously by generating target images when triggered, and maintaining high utility, signifying performance similar to clean models in generating high-quality images consistent with the training dataset distribution. Since current backdoor triggers all make poisoned images visually different from clean images and universal across all the inputs, they can be detected easily by universal perturbation-based detection and human inspection. In this paper, we aims to learn input-aware invisible triggers to make the injected backdoor stronger and more stealthy.
\vspace{-1em}
\subsection{Optimization framework for learnable invisible trigger}

\vspace{-0.7em}

In this paper, our goal is to inject invisible triggers into both unconditional and conditional diffusion models, enhancing the backdoor's stealthiness and effectiveness. To address this, we formulate the task as a bi-level optimization problem, which learns the invisible trigger to be inserted given different priors. For the purpose of generality, let $g$ be a trigger generator that generates invisible triggers given different priors $P$, $\mathcal{A}$ be the trigger insertion function which inserts the invisible trigger generated by $g$ into $P$, and $\bm{y}$ be the target image that the backdoored diffusion model will generate when the trigger is activated. Let $\bm{\epsilon}_\theta$ be the diffusion model which takes the prior $P$ as input to predict the noise, and $\mathcal{S}$ denote the whole sampling process of diffusion models which takes diffusion model $\bm{\epsilon}_\theta$ and $P$ as input to generate real data by iterative sampling. We first show the general optimization framework for different priors (i.e., unconditional and conditional generation). Then we will elaborate on the specific parameterization of $g$ and $\mathcal{A}$ for different priors $P$ in detail.

We formulate learning invisible backdoor for difussion model into a bi-level optimization problem. In the inner optimization, we optimize the trigger generator $g$ given fixed $\bm\epsilon_\theta$ to generate the target image $\bm{y}$. The MSE(mean squared error) is used as loss function to train $g$ as:
\begin{align}
    L_{inner}(\bm{\epsilon}_\theta, P, g(P)) = \left\|\mathcal{S} \Bigl(\bm{\epsilon}_\theta, \mathcal{A} (P, g(P)) \Bigl) - \bm{y} \right\|^2.
\end{align}
At the same time, to ensure the invisibility of the generated trigger, the generated trigger is bounded by $\ell_p$ norm ($\ell_\infty$ used in this paper), where we use PGD~\cite{madry2017towards} optimization to force the constraint. 
To optimize $g$, all intermediate results during sampling have to be stored to compute the gradient with respect to $g$. It thus becomes infeasible to sample many steps like the original DDPM~\cite{ho2020denoising} sampling. Alternatively, we use DDIM~\cite{song2020denoising} to perform the accelerated sampling process to make the optimization tractable. 

For the outer optimization of bi-level optimization, we optimize the diffusion model $\bm{\epsilon}_\theta$ to correctly predict the noise for both clean data and poisoned data (i.e., backdoor diffusion process).
Let $L_{outer}\left(\bm{\epsilon}_\theta, \bm{x}_0, P, g(P), \bm{y}, t\right)$ denote the loss function for the outer optimization,
where $L_{outer}$ is used to enforce the backdoor trigger's efficacy and model's utility. We defer the detailed $L_{outer}$ design in the following section based on different priors in condition and uncondition case.

Therefore, the whole bi-level optimization framework for input-aware invisible trigger can be formulated as
\begin{equation}
\begin{aligned}
\label{eq:framework}
    \min_\theta \; &L_{outer}\left(\bm{\epsilon}_\theta, \bm{x}_0, P, g^*_\theta(P), \bm{y}, t\right) \\
    \text{s.t.} \quad &g^*_\theta = \mathop{\arg \min}_g \left\|\mathcal{S} \Bigl(\bm{\epsilon}_\theta, \mathcal{A} (P, g(P) ) \Bigl) - \bm{y} \right\|^2, \|g(P)\|_\infty \leq C
\end{aligned}   
\end{equation}
\vspace{-0.7em}

where $\|\cdot\|_\infty$ denotes the $\ell_\infty$ norm, and $C$ is the norm bound of the trigger. Given the general framework in Equation~\ref{eq:framework}, we now describe the specific parameterization of the loss used under both unconditioned and conditioned scenarios.

\vspace{-1em}
\subsubsection{Backdooring unconditional diffusion models} 
\vspace{-0.7em}
When we backdoor unconditional diffusion model, the prior would be regarded as random Gaussian noise $P=\bm{\epsilon}, \bm{\epsilon} \sim \mathcal{N}(\bm{0}, \bm{I})$. 
Specifically, different from visible trigger used in prior works~\cite{chou2023backdoor,chou2023villandiffusion,chen2023trojdiff} on backdooring unconditional diffusion model, we optimize $g$ as a universal trigger generator for any random noise sampled from $\mathcal{N}(\bm{0}, \bm{I})$. That is, $g(\bm{\epsilon})=\bm{\delta}, \forall \bm{\epsilon} \sim \mathcal{N}(\bm{0}, \bm{I})$. The trigger injection function $\mathcal{A}$ then could be defined as $\mathcal{A}(P, g(P)) = \mathcal{A}(\bm{\epsilon}, \bm{\delta})=\bm{\delta} + \bm{\epsilon}$. In other words, we aim to make the diffusion model generate a specific image (target image) given any poisoned noise sampled from $\mathcal{N}(\bm{\delta}, \bm{I})$. By generating different $\bm{\delta}$, our proposed method could insert multiple invisible universal trigger-target pairs simultaneously. Furthermore, to make the trigger more invisible and versatile, instead of keeping $\bm{\delta}$ universal, we optimize $g$ to be a trigger distribution $\mathcal{N}(\bm{\delta}, \bm{I})$ so that any noise draw from the trigger distribution will make the diffusion model generate target image. Formally, we make $g(\bm{\epsilon})=\bm{\delta}'$ and $\mathcal{A}(P, g(P)) = \bm{\delta}' + \bm{\epsilon}, \bm{\delta}' \sim \mathcal{N}(\bm{\delta}, \bm{I}), \bm{\epsilon} \sim \mathcal{N}(\bm{0}, \bm{I})$.
Keeping the trigger invisible, our trigger generator is dynamic and sample-specific, which is able to bypass the current universal perturbation-based detection and defense.


Given the above parameterization of $g$ and $\mathcal{A}$, we define the loss function $L_{outer}$ based on DDIM~\cite{song2020denoising}
where we aim to create a secrect mapping for poisoned data between target image $y$ and poisoned distribution $\mathcal{N}(\bm{\delta}, \bm{I})$ or $\mathcal{N}(\bm{\delta}', \bm{I})$. 
Let $\bm{x}'_0 = \bm{y}$ (the target image distribution). $\bm{x}'_1, \bm{x}'_2, \cdots, \bm{x}'_T \in \mathbb{R}^d$ is then generated by gradually adding noise into $\bm{x}'_0$ where noise schedule is also controlled by $\beta_t$. The backdoored forward process is defined as:
\vspace{-1em}
\begin{align}
\label{eq:ddim_backdoor}
q_\sigma (\bm{x}'_{1:T}|\bm{x}'_0) &:= q_\sigma(\bm{x}'_T|\bm{x}'_0) \prod_{t=2}^T q_\sigma (\bm{x}'_{t-1}|\bm{x}'_t,\bm{x}'_0),
\end{align}
where $q_\sigma (\bm{x}'_T|\bm{x}'_0) = \mathcal{N}(\sqrt{\bar{\alpha}_T} \bm{x}'_0, (1-\bar{\alpha}_T) \bm{I})$, and $q_\sigma (\bm{x}'_{t-1}|\bm{x}'_t, \bm{x}'_0) = \mathcal{N}(\sqrt{\bar{\alpha}_{t-1}} \bm{x}'_0 + (1-\sqrt{\bar{\alpha}_{t-1}}) \bm{\delta} + \sqrt{1-\bar{\alpha}_{t-1}-\sigma_t^2} \cdot \frac{\bm{x}'_t - \sqrt{\bar{\alpha}_t} \bm{x}'_0 - (1-\sqrt{\bar{\alpha}_t}) \bm{\delta}}{\sqrt{1-\bar{\alpha}_t}}, \sigma_t^2 \bm{I})$.

The mean function in $q_\sigma (\bm{x}'_{t-1}|\bm{x}'_t, \bm{x}'_0)$ is chosen to ensure that $q_\sigma(\bm{x}'_t|\bm{x}'_0) := \mathcal{N}(\bm{x}'_t; \sqrt{\bar{\alpha}_t} \bm{x}'_0 + (1-\sqrt{\bar{\alpha}_t}) \bm{\delta}, (1-\bar{\alpha}_t) \bm{I})$ (See Appendix~\ref{sec:loss} for derivation). 
Using DDIM sampling process, we can directly set $\sigma_t=0$ to further simplify the derivation. 
the  loss function based on the minimization of KL divergence between parameterized $p_\theta(\bm{x}'_{t-1}|\bm{x}'_t)$ and $q_\sigma (\bm{x}'_{t-1}|\bm{x}'_t, \bm{x}'_0)$ can be written as
\vspace{-0.7em}
\begin{align}
\label{eq:loss}
    \mathop\mathbb{E}_{\bm{x}'_0, \bm{\epsilon}, t} \bigg[\|\bm{\epsilon} + \zeta_t\bm{\delta} 
    -\bm{\epsilon}_\theta(\bm{x}'_t(\bm{x}'_0, \bm{\delta}, \bm{\epsilon}), t)\|^2\bigg],
\end{align}
\vspace{-0.7em}

where $\zeta_t=\frac{\sqrt{\bar{\alpha}_{t-1}} - \sqrt{\bar{\alpha}_t}}{\sqrt{\bar{\alpha}_{t-1}} \sqrt{1-\bar{\alpha}_t} - \sqrt{\bar{\alpha}_t} \sqrt{1-\bar{\alpha}_{t-1}}}$, $\bm{x}'_t(\bm{x}'_0, \bm{\delta}, \bm{\epsilon})=\sqrt{\bar{\alpha}_t} \bm{x}'_0 + (1-\sqrt{\bar{\alpha}_t}) \bm{\delta} + \sqrt{1-\bar{\alpha}_t} \bm{\epsilon}$. We defer the full derivation in Appendix~\ref{sec:loss}. 
To let the diffusion model learn different distribution mapping, we construct the poisoned dataset as $\mathcal{D}=\{\mathcal{D}_c, \mathcal{D}_p\}$ where $\mathcal{D}_c$ denotes the clean data and $\mathcal{D}_p$ is the poisoned data. 
Now we can combine the loss function for backdooring diffusion process with loss function for clean diffusion process to obtain $L_{outer}$ for the outer optimization. The training algorithm under the bi-level optimization framework is shown in Algorithm~\ref{alg:unconditional}. 
During training, for poisoned sample and clean sample, we design the following loss function:
\vspace{-0.7em}
\begin{equation}
\begin{aligned}
    L_{outer}\left(\bm{\epsilon}_\theta, \bm{x}_0, \bm{\epsilon}, \bm{\delta}, \bm{y}, t\right) = 
    \begin{cases}
        &\|\bm{\epsilon} - \bm{\epsilon}_\theta(\sqrt{\bar{\alpha}_t} \bm{x}_0 + \sqrt{1 - \bar{\alpha}_t} \bm{\epsilon}, t)\|^2, \; \; \text{if} \; \bm{x}_0 \in \mathcal{D}_c, \\
        &\|\bm{\epsilon} +\zeta_t\bm{\delta}
        -\bm{\epsilon}_\theta(\bm{x}'_t(\bm{y}, \bm{\delta}, \bm{\epsilon}), t)\|^2, \; \;   \text{if} \;  \bm{x}_0 \in \mathcal{D}_p.
    \end{cases}
\end{aligned}    
\end{equation}
Moreover, rather than being limited to only use a specific sampler~\cite{chou2023backdoor}, our proposed framework could choose a wide range of samplers such as DDIM~\cite{song2020denoising}, DPMSolver~\cite{lu2022dpm} etc.



\vspace{-1em}
\subsubsection{Backdooring conditional diffusion models}
\vspace{-0.7em}
\label{sec:conditional}
Different from the above unconditional diffusion models in which the prior is random noise, conditional diffusion models can have various priors, such as texts, images, masked images, and even sketches~\cite{ramesh2021zero, nichol2021glide, meng2021sdedit, saharia2022photorealistic, rombach2022high}. Since the prior are all natural images and texts, it is thus crucial to make the trigger invisible where previously used visible triggers~\cite{chou2023backdoor, chen2023trojdiff, chou2023villandiffusion} can be detected without any effort.
For simplicity, we formulate how to learn invisible triggers for conditional diffusion model in the text-guided image editing pipeline used in~\cite{nichol2021glide}.  For text-guided image editing, the priors consist of the masked image to be edited, a mask marking editing regions, and text~\cite{nichol2021glide, rombach2022high}.  Unlike previous works~\cite{struppek2022rickrolling, chou2023villandiffusion} that insert the backdoor into text representation, to best of our knowledge, not only are we the first to propose a general framework to learn invisible triggers but also the first to show how to backdoor text-guided image editing/inpainting pipeline. 

Let the masked image be $\tilde{\bm{x}}=\bm{x}_0 \odot \bm{M}$ where $\bm{M}$ is the binary mask. Let $c$ be the text instruction for editing. Then the priors can be written as $P=\{\tilde{\bm{x}}, \bm{M}, c\}$. The aim of invisible triggers in this pipeline is to only insert imperceptible perturbation into masked natural images $\tilde{\bm{x}}$ to backdoor diffusion models. 
In this setting, we have to ensure that there is no perturbation/trigger in the masked region, or the inserted trigger can be immediately detected since the pixel values must be 0 in the masked region. 
To this end, we parameterize the trigger generator $g$ with a simple neural network to learn input-aware triggers given masked image $\tilde{\bm{x}}$, mask $\bm{M}$, and target image $\bm{y}$. Let $\bm\delta^M_{\bm{x}_0} = g(\tilde{\bm{x}}, \bm{M}, \bm{y})$ be the generated input-aware triggers. To ensure there is no perturbation in the masked region, we directly constrain the generated trigger to have zero value on the masked region. 

We consider two kinds of masks, rectangular masks and free-form masks proposed in~\cite{yu2019free}, which can mimic the user-specified masks used in real-world applications. Note that due to the existence of additional priors, the sampling process is different now. With the additional priors, classifier-free guidance is used to generate images conditioned on additional priors~\cite{ho2022classifier, nichol2021glide}. The predicted noise is computed as $(1-\gamma)\bm\epsilon_\theta(\bm{x}_t, t, \tilde{\bm{x}}, \bm{M}, \emptyset) + \gamma\bm\epsilon_\theta(\bm{x}_t, t, \tilde{\bm{x}}, \bm{M}, c)$ where $\gamma$ is a hyperparameter that controls the strength of guidance~\cite{ho2022classifier}. 


As defined in the threat model, the attacker aims to generate the target image $\bm{y}$ when the backdoor trigger is activated. This asks the conditional diffusion model should output the same target image regardless of any given text. In other words, the diffusion models should predict the same noise whenever there are triggers in the masked image. By setting the text to be a empty string, we mimic the aforementioned procedure as $\bm\epsilon_\theta(\bm{x}_t, t, \tilde{\bm{x}}+\bm\delta^M_{\bm{x}_0}, \bm{M}, c)=\bm\epsilon_\theta(\bm{x}_t, t, \tilde{\bm{x}}+\bm\delta^M_{\bm{x}_0}, \bm{M}, \emptyset)$.  Therefore, in the inner optimization, we perform the sampling process independently with $c$ or $\emptyset$, similar to the above unconditional sampling instead of classifier-free guidance sampling, to generate target image. 



We summarize the training algorithm to insert backdoor in conditional diffusion models in Algorithm~\ref{alg:conditional}.
Given $\bm{x}_0 \sim \{\mathcal{D}_c, \mathcal{D}_p\}$, for the clean training (i.e., $\bm{x}_0 \in \mathcal{D}_c$), we add noise to $\bm{x}_0$ and optimize $\bm{\epsilon}_\theta$ which takes noisy $\bm{x}_0, \bm{x}_0 \odot \bm{M}, \bm{M}$ and $c$ as input to predict noise. For the backdoor training (i.e., $\bm{x}_0 \in \mathcal{D}_p$), we firstly sample clean data $\bm{x}_c \sim \mathcal{D}_c$, compute the corresponding masked version as $(\bm{x}_c \odot \bm{M})$, and generate the injected trigger $\bm{\delta}^M_{\bm{x}_c}$ for the masked image. Then we add noise to target image $\bm{y}$, and optimize $\bm{\epsilon}_\theta$ which takes noisy $\bm{y}, (\bm{x}_c \odot \bm{M}+\delta^M_{\bm{x}_c}), \bm{M}$, and $c$ as input, to predict the noise. Hence $L_{outer}$ for the outer optimization can be written as
\vspace{-0.5em}
\begin{equation}
\begin{aligned}
    L_{outer}\left(\bm{\epsilon}_\theta, \bm{x}_0, \bm{M}, c, \bm\delta^M_{\bm{x}_c}, \bm{y}, t\right) =
    \begin{cases}
        \|\bm{\epsilon} - \bm{\epsilon}_\theta(\bm{x}_t, t, \bm{\tilde{x}}_0, \bm{M}, c)\|^2, 
        \text{if} \; \bm{x}_0 \in \mathcal{D}_c, \\
        \|\bm{\epsilon} -\bm{\epsilon}_\theta(\bm{y}_t, t, \bm{\tilde{x}}_c
        +\bm\delta^M_{\bm{x}_c}, \bm{M}, c)\|^2, \text{if} \; \bm{x}_0 \in \mathcal{D}_p,
    \end{cases}
\end{aligned}    
\end{equation}
where $\bm{x}_t=\sqrt{\bar{\alpha}_t} \bm{x}_0 + \sqrt{1 - \bar{\alpha}_t} \bm{\epsilon}, \bm{y}_t=\sqrt{\bar{\alpha}_t} \bm{y} + \sqrt{1 - \bar{\alpha}_t} \bm{\epsilon}$.  
\vspace{-0.7em}

\paragraph{Using invisible trigger as model watermarking} 
Because of the invisibility and robustness of our framework, we show invisible backdoors in conditional generation could be utilized into model watermarking for model ownership verification. In this setting, a model owner aims to insert invisible backdoors into the conditional diffusion model as watermarks using our proposed framework for model copyright protection. Given any inspected model, investigators aims to verify if the inspected model is derived from the watermarked model, where investigators only have the black-box access to the inspected model without its internal information. If the inspected model is derived from the watermarked model, then the output would be the target image given any input with the trigger; otherwise the output won't share much similarity with the target image. Hence the investigators are able to query the inspected model with input images with the triggers and then compute the MSE between the output images and the target image as a metric to know whether there is a misappropriation. 

\begin{algorithm}[t]
	\caption{Backdoored diffusion model training given the priors are masked image, mask, and text, i.e., conditional generation.}  
	\label{alg:conditional}
	\begin{algorithmic}[1]
		\STATE {\bfseries Input:} $K,D$, stepsizes $\alpha$ and $\beta$, initializations $g_0$ and $\theta_0$, target image $\bm{y}$, dataset $\mathcal{D}=\{\mathcal{D}_c, \mathcal{D}_p\}$, function GenerateRandomMask() for generating random masks. 
		\FOR{$k=0,1,2,...,K$}
		\STATE{Set $g_k^0 = g_{k-1}^{D} \mbox{ if }\; k> 0$ and $g_0$ otherwise 
		}
		\FOR{$i=1,....,D$}
		\STATE{$\bm{M}' \leftarrow$} GenerateRandomMask()  
            \STATE{$(\bm{x}_0, c) \sim \{\mathcal{D}_c, \mathcal{D}_p$\}. Set $c=\emptyset$ with probability $50\%$.}
            \STATE{$\tilde{\bm\delta}_{\bm{x}_0, i}^M = g_k^{i-1}(\bm{x}_0 \odot \bm{M}', \bm{M}', \bm{y}) \odot \bm{M}'$}
            \STATE{$\bm\delta_{\bm{x}_0, i}^M = \mathrm{Proj}_{\|\cdot\|_\infty \leq C} (\tilde{\bm\delta}_{\bm{x}_0, i}^M)$}
		\STATE{$g_k^i = g_k^{i-1}-\alpha \nabla_{g} L_{inner}(\bm{\epsilon}_{\theta_k},\bm{x}_0 \odot \bm{M}', \bm{M}', c, \bm\delta_{\bm{x}_0, i}^M)$}
		\ENDFOR
                 \STATE{$(\bm{x}_0, c) \sim \{\mathcal{D}_c, \mathcal{D}_p$\}. Set $c=\emptyset$ with probability $50\%$.}
                 \STATE{$t \sim \text{Uniform}(\{1, \ldots, T\})$}
                 \STATE{$\bm{M} \leftarrow$} GenerateRandomMask()
                 \IF{$\bm{x}_0 \in \mathcal{D}_p$}
                 \STATE{Sample $\bm{x}_c \sim \mathcal{D}_c$}
                 \STATE{$\tilde{\bm\delta}^M_{\bm{x}_c} = g_k^D(\bm{x}_c \odot \bm{M}, \bm{M}, \bm{y}) \odot \bm{M}$}

                 \STATE{$\bm\delta_{\bm{x}_c}^M = \mathrm{Proj}_{\|\cdot\|_\infty \leq C} (\tilde{\bm\delta}_{\bm{x}_c}^M)$}
                 \ENDIF

                  \STATE{Get $\nabla_\theta L =\nabla_\theta L_{outer}\left(\bm{\epsilon}_\theta, \bm{x}_0, \bm{\tilde{x}_0}, \bm{M}, c, \bm\delta^M_{\bm{x}_c}, \bm{y}, t\right)$}
                 \STATE{$\theta_{k+1}=\theta_k- \beta \nabla_\theta L$
        }
		\ENDFOR
	\end{algorithmic}
	\end{algorithm}	

\vspace{-1em}
\section{Experiments}
\label{sec:exps}
\vspace{-1em}
\subsection{Implementation details}
\vspace{-0.7em}
For unconditional generation, we conduct the experiments on two commonly used datasets, CIFAR10($32\times 32$)~\cite{krizhevsky2009learning} and CELEBA-HQ($256\times 256$)~\cite{liu2015deep} used in ~\cite{chou2023backdoor, chou2023villandiffusion}. For conditional generation, we follow the text-guided image editing/inpainting pipeline in~\cite{nichol2021glide} and use the dataset MS COCO($64\times 64$)~\cite{lin2014microsoft}. The diffusion models are trained from scratch for 400 epochs on both CIFAR10 and CELEBA-HQ for unconditional generation and we also show that finetuning pre-trained models with less epochs is also feasible to inject the proposed invisible backdoor. For conditional case, we found that only finetuning for 5 epochs on about 10K images of MS COCO training data is enough to learn input-aware invisible backdoor. The learning rate of inner optimization is $1e-3$ for all cases, and for outer optimization, the learning rates are $2e-4$, $8e-5$, and $5e-4$ for CIFAR10, CELEBA-HQ, and MS COCO, respectively. To make inner optimization feasible, we sample 10, 3, and 5 steps to generate target images with DDIM~\cite{song2020denoising} sampling for CIFAR10, CELEBA-HQ, and MS COCO, respectively. All unconditional generation experiments are conducted on a single NVIDIA 3090 GPU, and all conditional generation experiments are conducted on a single NVIDIA A6000 GPU. Training on CIFAR10 and CELEBA-HQ from scratch spends about 3 days and 12 days respectively due to the large image resolution of CELEBA-HQ and multiple sampling steps in the inner optimization, which can be accelerated by more powerful GPUs. Training on MS COCO could be finished within about 30 mins since it only requires finetuning for 5 epochs. We mainly use three target images corresponding to the `Hat', `Shoe', and `Cat' target used in~\cite{chou2023backdoor, chou2023villandiffusion}. To evaluate the performance of backdoored model on utility and specificity, for unconditional case, we use FID~\cite{heusel2017gans} to evaluate the utility, and MSE to evaluate the specificity. We sample about 50K and 10K images to compute FID and MSE for CIFAR10 and CELEBA-HQ, respectively. For conditional case, we use FID and LPIPS~\cite{zhang2018unreasonable} to evaluate the utility and MSE to evaluate the specificity. We sample the same number of images as the MS COCO validation set (about 40K) to compute FID, LPIPS, and MSE.


\vspace{-1em}
\subsection{Unconditional generation results}
\vspace{-0.7em}
\paragraph{Universal backdoor triggers}
Firstly, we show the results for learning one universal invisible trigger on CIFAR10 and CELEBA-HQ. For CIFAR10, $\ell_\infty$ norm bound is set as 0.2 and the poison rate is 0.05. The `Hat' image is the target image. As shown in Figure~\ref{fig:visualize_universal} and Table~\ref{tab:visualize_universal}, the backdoored model can achieve similar FID with clean model and low MSE simultaneously while keeping the trigger invisible, demonstrating the high-utility and high-specificity as required by successful backdoor attack. 
\begin{figure}[h]
  \centering
  \vspace{-0.7em}
  \begin{minipage}{0.312\textwidth}
    \centering
    \includegraphics[width=\linewidth]{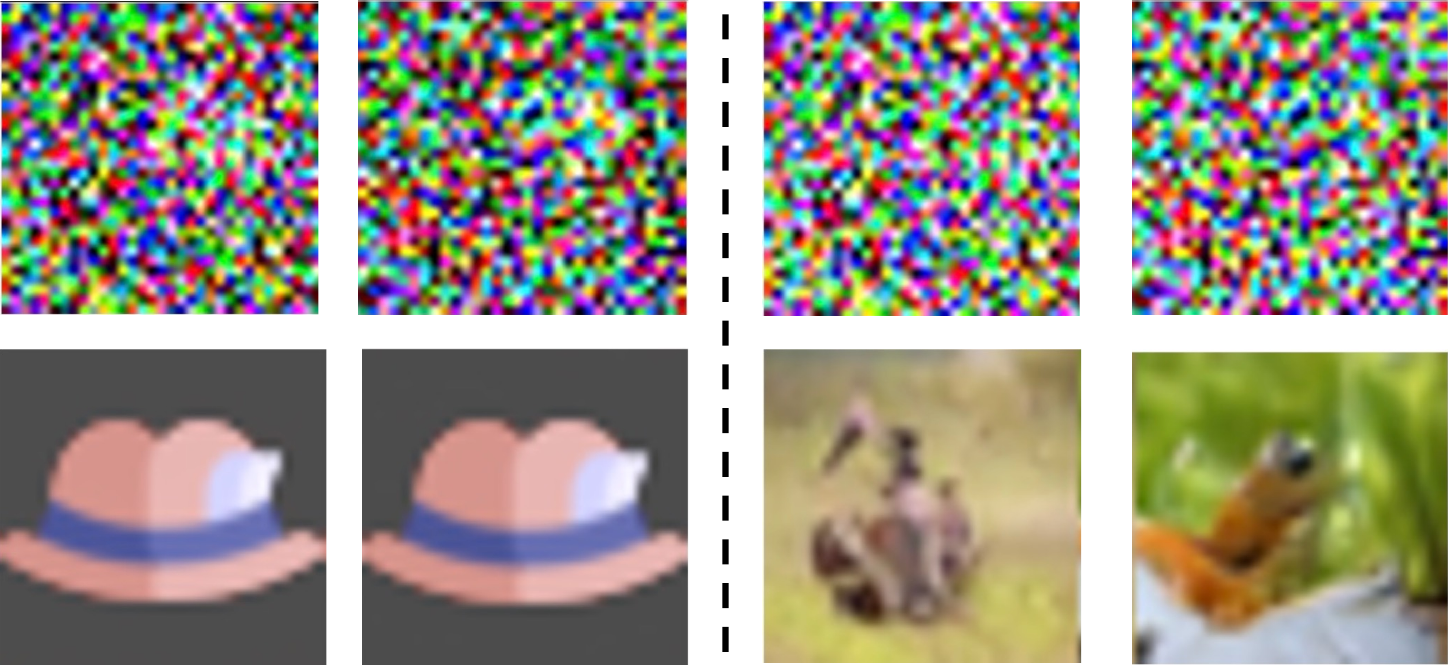} 
    \caption{Examples of learnable invisible universal trigger on CIFAR10.}
    \label{fig:visualize_universal}
  \end{minipage}\hfill
  \begin{minipage}{0.3\textwidth}
    \centering
    \includegraphics[width=\linewidth]{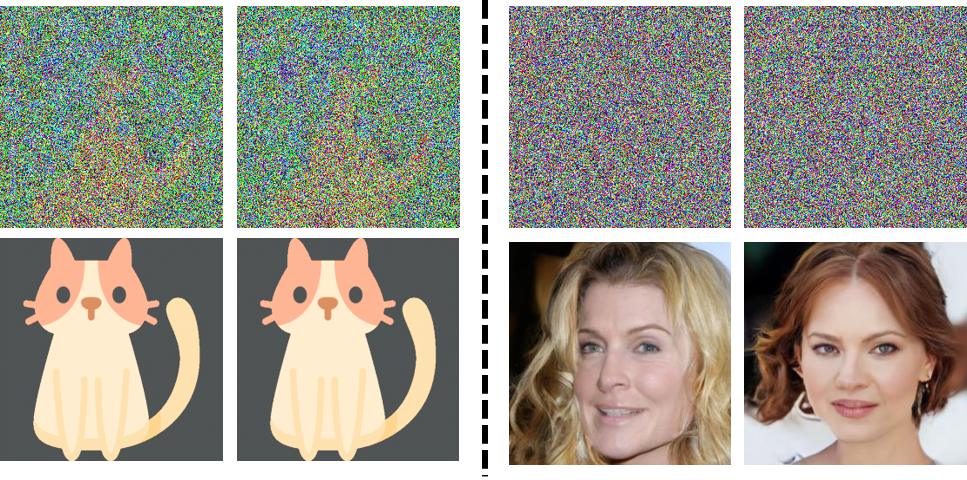} 
    \caption{Examples of universal trigger on high-resolution dataset, CELEBA-HQ.}
    \label{fig:visualize_universal1}
  \end{minipage}\hfill
  \begin{minipage}{0.3\textwidth}
\centering
\includegraphics[width=\linewidth]{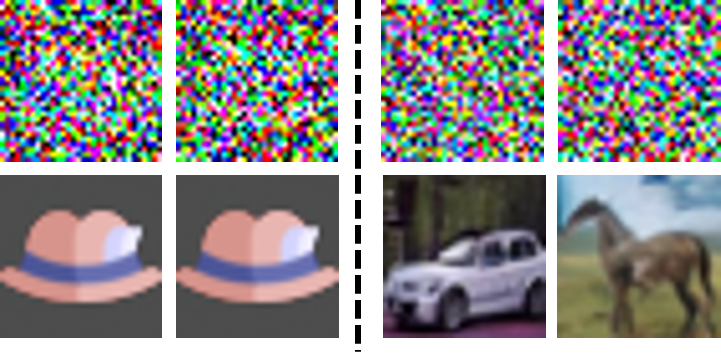}
\caption{Examples of learnable trigger distribution on CIFAR10.}
\label{fig:visualize_distribution}
\end{minipage}
\vspace{-1em}
\end{figure}
\begin{table}[ht]
\centering
\vspace{-1em}
\caption{FID and MSE results of the universal and sample-specific triggers on CIFAR10 and CELEBA-HQ, demonstrating high-utility and high-specificity.}
\begin{tabular}{c|c|c|c|c}
\hline
Trigger type & Dataset & Model type & FID & MSE \\
\hline
\multirow{4}{*}{Universal Trigger} & \multirow{2}{*}{CIFAR10} & Clean model & 12.80 & - \\
& & Backdoored model & 11.76 & 3.07e-3 \\
\cline{2-5}
& \multirow{2}{*}{CELEBA-HQ} & Clean model & 12.39 & - \\
& & Backdoored model & 11.19 & 4.57e-3 \\
\hline
\multirow{2}{*}{Sample-specific Trigger} & \multirow{2}{*}{CIFAR10} & Clean model & 12.80 & - \\
& & Backdoored model & 12.86 & 1.82e-5 \\
\hline
\end{tabular}
\label{tab:visualize_universal}
\end{table}
We further show the results on high-resolution datasets CELEBA-HQ($256 \times 256$) in Figure~\ref{fig:visualize_universal1} and Table~\ref{tab:visualize_universal}, where the $\ell_\infty$ norm bound is also 0.2 and the poison rate is 0.3. The `Cat' image is the corresponding target image. It could clearly observed that the proposed invisible trigger is still highly effective for high-resolution images.

\begin{wrapfigure}{r}{0.6\textwidth}
  \begin{center}
    \includegraphics[width=0.5\textwidth]{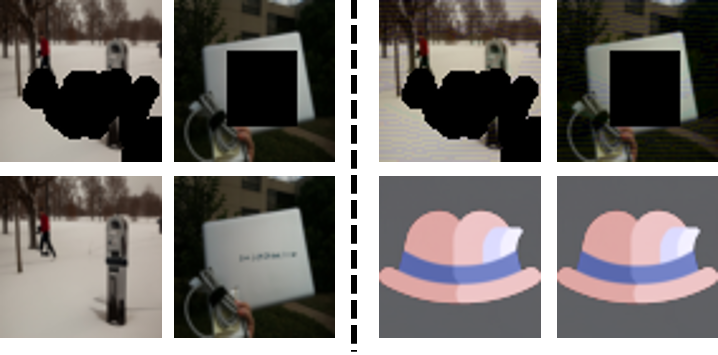}
  \end{center}
  \vspace{-1em}
  \caption{Examples of invisible input-aware trigger in conditional diffusion models. Given the masked image, the conditional diffusion will perform normally and edit the masked image following text description if there is no trigger inside. However, if the invisible trigger is inserted into the masked image, the model will output the target image regardless of any given text.}
  \label{fig:condition}
  \vspace{-1em}
\end{wrapfigure}
To further show the capability of the proposed framework, we show it is possible to learn multiple universal trigger-target pairs simultaneously. Examples are available in Appendix~\ref{sec:uncondition_multiple}. 
In addition, we also conducted experiments on different samplers, DPMSolver~\cite{lu2022dpm} to show that the proposed loss in unconditional generation can be directly applied to different commonly used samplers. Detailed results are shown in Appendix~\ref{sec:samplers}.

\vspace{-1em}
\paragraph{Distribution based trigger results}
As stated in Section~\ref{sec:method}, to make the invisible trigger even more stealthy, we can optimize the trigger distribution instead of universal trigger so that we are able to generate sample-specific triggers. The results are shown in Figure~\ref{fig:visualize_distribution} and Table~\ref{tab:visualize_universal}. Compared with universal trigger, our distribution based trigger achieve a even smaller gap on the FID while keeping an excellent performance on generating target image with very small MSE.


\vspace{-1em}
\subsection{Conditional generation results}
\vspace{-1em}
In this section, we show the results in text-guided image editing/inpainting pipeline on MS COCO dataset~\cite{lin2014microsoft}. For simplicity, we ignore the text part in visualization results since the proposed framework will generate the target image given any text if the backdoor is triggered. We randomly mask part of the clean images and send the masked images to the trigger generator for inserting input-aware triggers. By setting the norm bound as 0.04, we show the quantitative results on evaluation metrics and visualization results in Figure~\ref{fig:condition} and Table~\ref{tab:condition}, where the target image is the `Hat' image. As shown in Figure~\ref{fig:condition}, images with any shape of masks with our triggers will lead to the target `Hat' image, while image without the trigger would inpaint the image naturally. The results indicate that the backdoored model perform similarly to clean model when there are no triggers in the inputs, and generate target image when triggers are injected into inputs. 


\begin{table}[htbp]
\vspace{-1em}
\centering
\caption{Results on different evaluation metrics for learnable input-aware trigger in conditional diffusion models, which show the stealthiness and effectiveness of the attack, with no effect on the clean performance.}
\begin{tabular}{c|c|c|c}
\hline
& FID & LPIPS & MSE \\
\hline
Clean model & 1.00 & 0.064 & - \\
Backdoored model & 1.01 & 0.064 & 6.85e-3 \\
\hline
\end{tabular}
\label{tab:condition}
\end{table}

Figure~\ref{fig:visualize_norm_condition} shows the visualization results under different norm bounds. With larger norm bound, we can expect larger perturbations in the generated triggers, which is illustrated in the visualization results.
\begin{figure}[htbp]
    \centering
    \vspace{-1em}
    \includegraphics[width=1.0\linewidth]{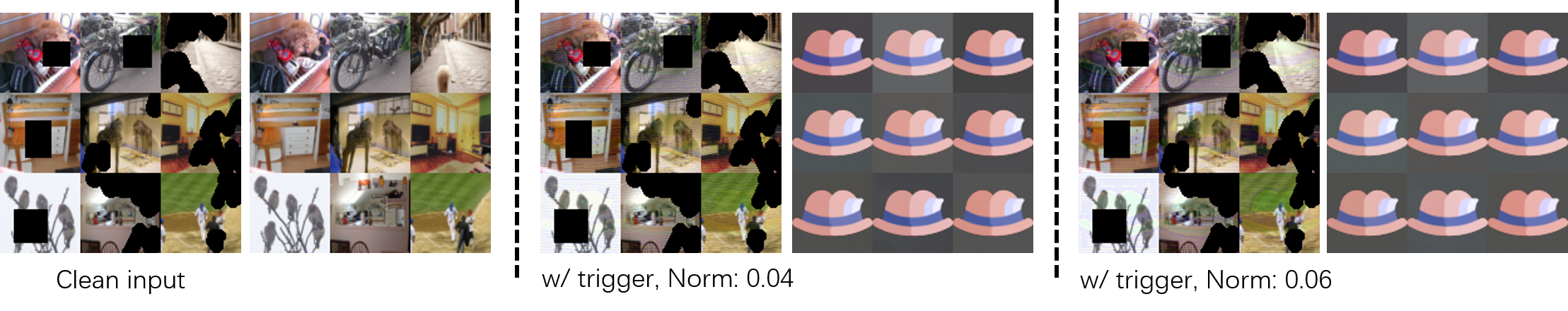}
    \caption{Visualization under different norm bounds, 0.04 and 0.06, with backdoored text-guided image editing model.}
    \label{fig:visualize_norm_condition}
    \vspace{-1em}
\end{figure}
We also conduct experiments to show it is possible to insert multiple targets in this case. Please refer to Appendix~\ref{sec:multiple_condition} for details.
\vspace{-1em}

\paragraph{Model watermarking}
\label{sec:watermarking}

As discussed in Section~\ref{sec:conditional}, our proposed backdoor trigger could be further developed into a watermark framework for conditional diffusion model (text-guided image editing/inpainting in our considered case).  To show the effectiveness of using our invisible backdoor as watermarks, we use the `Hat' image as the target image and insert the backdoor (watermark) into the conditional diffusion model. 
Table~\ref{tab:watermark} shows the MSE results between the outputs images and the target image for different query times. It could be observed that even with only 50 queries, there alreday exists a huge MSE gap between watermarked model and non-watermarked model so that we could just set threshold to 0.1 to decide if a model is derived from the watermarked one. 
\begin{table}[htbp]
\vspace{-1em}
\centering
\caption{MSE results between the outputs images and the target image for the watermarked model and non-watermarked model, under different number of queries. The results show that even with only 50 queries, the watermarked model can be differentiated from the non-watermarked model.}
\begin{tabular}{c|c|c|c|c|c|c}
\hline
Num of queries & 50 & 200 & 500 & 1000 & Mean & Variance \\
\hline
Watermarked model & 2.22e-2 & 2.05e-2 & 3.65e-2 & 2.66e-2 & 2.64e-2 & 3.86e-5 \\
Non-Watermarked Model & 0.452 & 0.458 & 0.452 & 0.448 & 0.452 & 1.28e-5\\
\hline
\end{tabular}
\label{tab:watermark}
\end{table}

\vspace{-1em}

\subsection{Results under counter-measurements}
\label{subsec:defense}
\vspace{-0.7em}
Recent papers also propose various counter-measurements against backdoor attack in the diffusion model. To verify our proposed attack's robustness, we test various defense methods against the proposed framework. We first tried fine-tuning the backdoored model with clean data however we find it still couldn't mitigate the inserted backdoors. Secondly, as suggested by previous work~\cite{chou2023backdoor}, Adversarial Neuron Pruning~\cite{wu2021adversarial} and Inference-time Clipping are two most effective defense methods against backdoor in diffusion models. They show the inference-time clipping, which clips the latent generation during sampling to the range $[-1, 1]$, is effective on their proposed attack. However, we found that both of them become totally ineffective in our proposed framework. Thirdly, we test the most recent defense method Elijah~\cite{an2024elijah} which is specifically designed for backdoors in diffusion models. The results show that Elijah is also entirely ineffective in our framework, which indicates more advanced defense methods need to be developed for diffusion models. For detailed results, please refer to Appendix~\ref{sec:clip}. 

\vspace{-1em}

\subsection{Ablation study }
\label{subsubsec:ablation}
\vspace{-0.7em}
\paragraph{Norm bound} In the inner loop of the bi-level optimization, we project the generated triggers into an $\ell_\infty$ norm ball to ensure trigger invisibility. Here, we explore the impact of varying norm bounds on our model. The visualization results in Figure~\ref{fig:visualize_norm} illustrate how different norm values affect the unconditional generation case. 

\begin{table}[ht]
\vspace{-1em}
\centering
\begin{minipage}{0.49\linewidth}
\centering
\includegraphics[width=66mm]{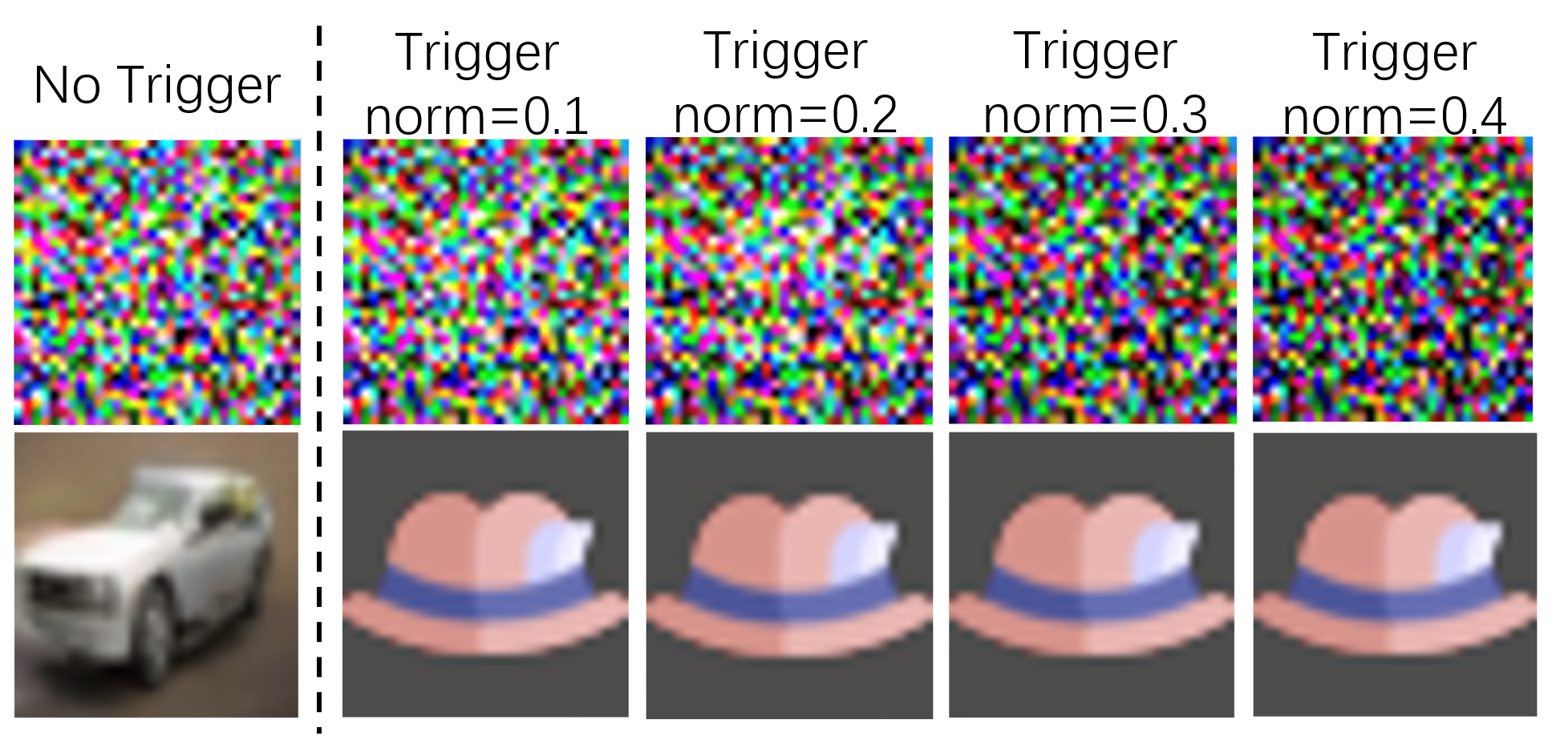}
\captionof{figure}{Illustration of the effect of different norm bounds on unconditional generation.}
\label{fig:visualize_norm}
\end{minipage}\hfill
\begin{minipage}{0.49\linewidth}
\centering
\caption{FID and MSE results with different norm bounds on CIFAR10.}
\begin{tabular}{c|c|c|c|c}
\hline
& \multicolumn{2}{c|}{Target `Hat'} & \multicolumn{2}{c}{Target `Cat'} \\ \hline
Norm & FID & MSE & FID & MSE \\ \hline
0.1 & 13.01 & 2.28e-3 & 12.92 & 2.75e-3 \\ \hline
0.2 & 12.44 & 8.13e-5 & 12.56 & 1.01e-5 \\ \hline
0.3 & 12.38 & 1.06e-3 & 12.69 & 6.00e-4 \\ \hline
0.4 & 12.35 & 1.92e-4 & 12.93 & 2.77e-6 \\ \hline
\end{tabular}
\label{tab:norm}
\end{minipage}
\end{table}

\vspace{-1em}

We also measure the generated target image on two different targets (`Hat' and `Cat') in Table~\ref{tab:norm}. The FID corresponding to clean model is 12.80, as shown in Table~\ref{tab:visualize_universal}.
Notably, even with a low norm value of 0.1 (indicating invisibility), our optimization framework successfully implants a backdoor with high specificity while maintaining a comparable FID score (utility) in contrast to the clean model, which has an FID of 12.80, as demonstrated in Table~\ref{tab:visualize_universal}. This finding is particularly insightful when considering the application of our framework in conditional settings, where trigger invisibility is pivotal for the stealthy integration of our implanted backdoor.

\vspace{-0.7em}
\paragraph{Poison rates} We performed experiments to demonstrate the impact of different poison rates, as illustrated in Figure~\ref{fig:fid_poisonrate} and~\ref{fig:mse_poisonrate}. As the poison rate increases, we observe that the FID score increases while the MSE (Mean Squared Error) decreases. This aligns with our expectations since a larger poison rate implies a more substantial impact on clean performance. When employing a high poison rate (e.g., poison rate of 0.5), we find that the FID remains comparable to that of the clean model. This observation suggests that the proposed framework maintains its effectiveness across a range of settings and is resilient even under substantial poisoning conditions.

\begin{figure}[htbp]
  \centering
  \begin{minipage}{0.45\textwidth}
    \centering
    \includegraphics[width=0.7\linewidth]{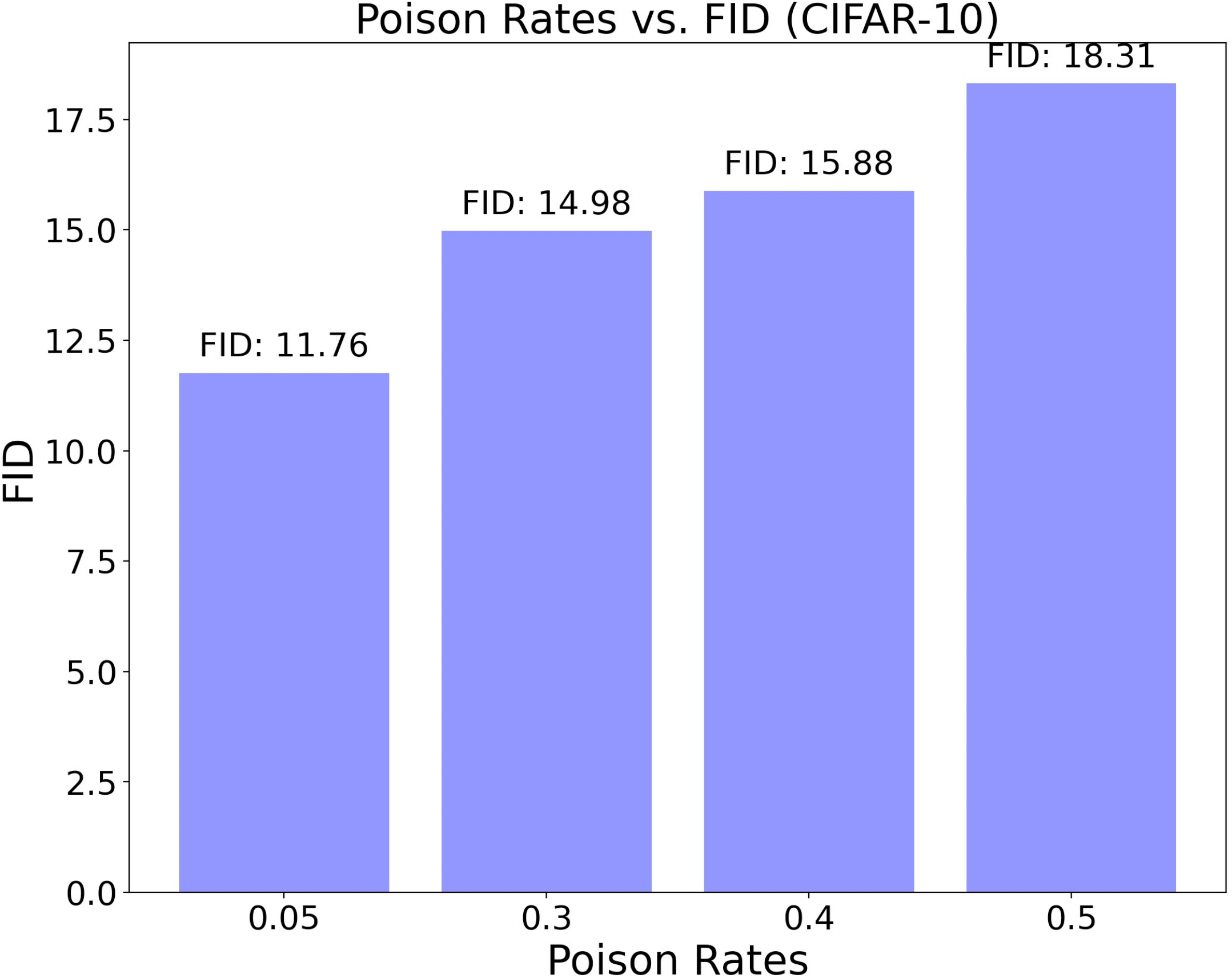} 
    \caption{FID results for different poison rates on CIFAR10.}
    \label{fig:fid_poisonrate}
  \end{minipage}\hfill
  \begin{minipage}{0.45\textwidth}
    \centering
    \includegraphics[width=0.7\linewidth]{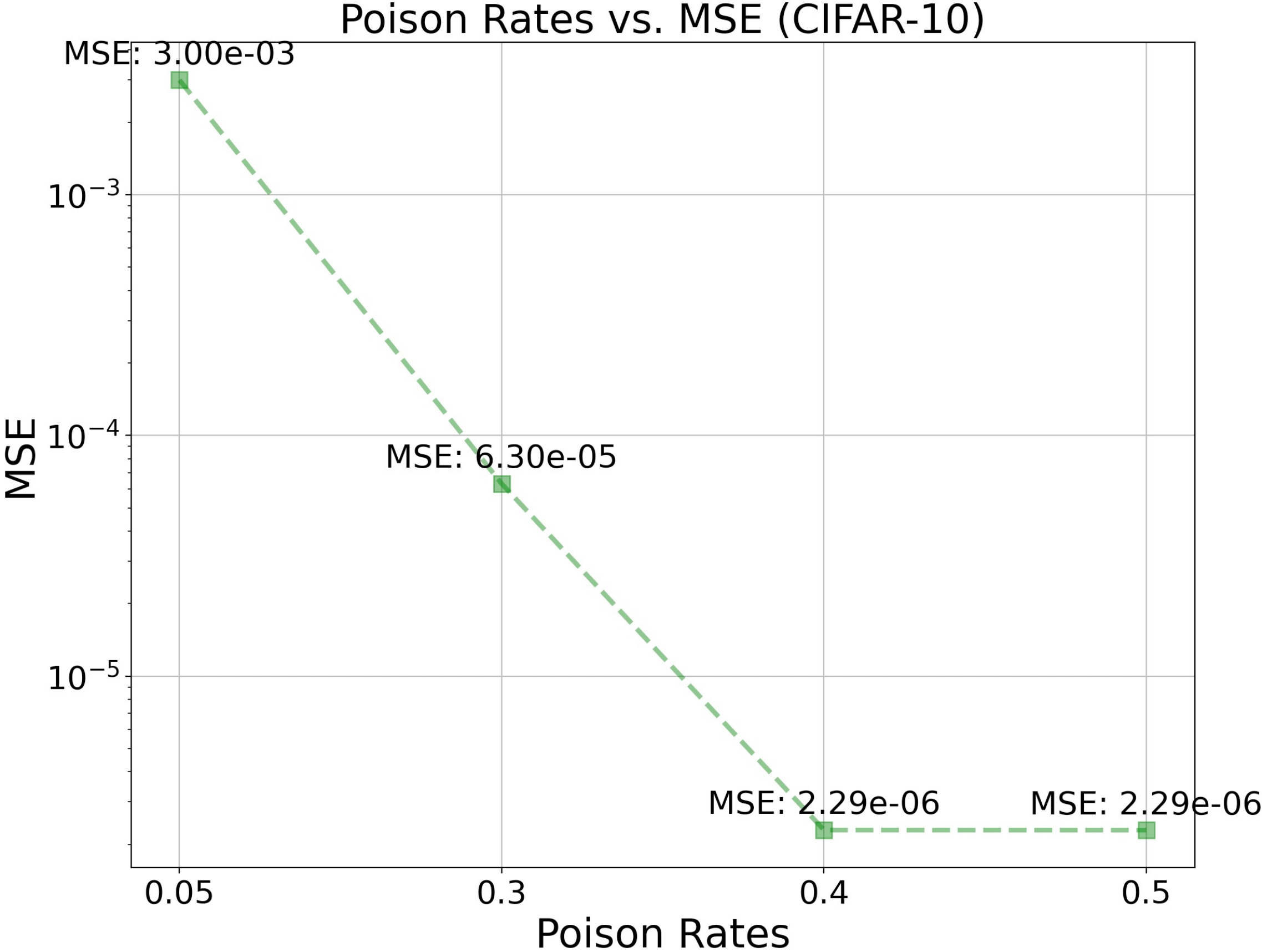} 
    \caption{MSE results for different poison rates on CIFAR10.}
    \label{fig:mse_poisonrate}
    \end{minipage}
    \vspace{-1em}
\end{figure}
Furthermore, we demonstrate that finetuing with less epochs is enough to effectively insert the backdoors, greatly reducing the time cost and making the proposed framework more practical. (Results deferred to Appendix~\ref{sec:finetune} due to space limit).

\vspace{-1em}
\section{Conclusion and Limitations}
\label{sec:conclusion}
\vspace{-1em}
In this paper, we introduce an innovative and versatile optimization framework designed to learn input-aware invisible triggers, enabling the backdooring of diffusion models applicable to both unconditional and conditional scenarios. 
Our work marks the pioneering demonstration of backdooring in the text-guided image editing pipeline within conditional diffusion models. The application to model watermarking further enhances its importance. Our proposed framework sheds light on the significant security threats posed by diffusion models, emphasizing the need for comprehensive exploration and understanding. Looking ahead, since the proposed framework involves bi-level optimization which is generally a little time-consuming, we will explore different strategies such as faster sampling to further accelerate the training or finetuning process in the future. Moreover, our future works will also focus on developing effective defense methods to mitigate potential backdoors in diffusion models, contributing to the advancement of secure and reliable implementations in various applications.




{
\small

\bibliographystyle{plain}
\bibliography{ref}
}


\newpage

\appendix

\section{Broader Impact}
\label{sec:broader}

Our work offers significant benefits to both the research community focusing on backdoor attacks and those engaged in industrial applications.

For the research community, we present a groundbreaking and potent backdoor attack, exposing a previously overlooked potential threat. By employing an invisible attack trigger, our novel approach easily circumvents human inspection, necessitating the development of more advanced defense methods to effectively mitigate future risks in research.

Regarding industrial applications, our findings enable model owners to consider the implications of our proposed attack and implement suitable protection strategies for enhanced deployment. Additionally, model users can now be mindful of the potential existence of such robust attacks in third-party models, allowing them to exercise greater caution when utilizing these models.

\section{Loss function based on DDIM sampling}
\label{sec:loss}

As we show in Equation~\ref{eq:ddim_backdoor}, $q_\sigma (\bm{x}'_{t-1}|\bm{x}'_t, \bm{x}'_0) = \mathcal{N}(\sqrt{\bar{\alpha}_{t-1}} \bm{x}'_0 + (1-\sqrt{\bar{\alpha}_{t-1}}) \bm{\delta} + \sqrt{1-\bar{\alpha}_{t-1}-\sigma_t^2} \cdot \frac{\bm{x}'_t - \sqrt{\bar{\alpha}_t} \bm{x}'_0 - (1-\sqrt{\bar{\alpha}_t}) \bm{\delta}}{\sqrt{1-\bar{\alpha}_t}}, \sigma_t^2 \bm{I})$,
where the mean function is chosen to ensure that $q_\sigma(\bm{x}'_t|\bm{x}'_0) = \mathcal{N}(\bm{x}'_t; \sqrt{\bar{\alpha}_t} \bm{x}'_0 + (1-\sqrt{\bar{\alpha}_t}) \bm{\delta}, (1-\bar{\alpha}_t) \bm{I})$. We provide the proof in the following.

\begin{lemma}
    Let $q_\sigma (\bm{x}'_{1:T}|\bm{x}'_0)$ and $q_\sigma (\bm{x}'_{t-1}|\bm{x}'_t, \bm{x}'_0)$ be defined by Equation~\ref{eq:ddim_backdoor}, we have
    \begin{align*}
        q_\sigma(\bm{x}'_t|\bm{x}'_0) = \mathcal{N}(\bm{x}'_t; \sqrt{\bar{\alpha}_t} \bm{x}'_0 + (1-\sqrt{\bar{\alpha}_t}) \bm{\delta}, (1-\bar{\alpha}_t) \bm{I}).
    \end{align*}
\end{lemma}
\begin{proof}
    To prove it, we use a similar way to~\cite{song2020denoising}. Assume $\forall t \leq T$, $q_\sigma(\bm{x}'_t|\bm{x}'_0) = \mathcal{N}(\sqrt{\bar{\alpha}_t} \bm{x}'_0 + (1-\sqrt{\bar{\alpha}_t}) \bm{\delta}, (1-\bar{\alpha}_t) \bm{I})$. Now if we can prove $q_\sigma(\bm{x}'_{t-1}|\bm{x}'_0) = \mathcal{N}(\sqrt{\bar{\alpha}_{t-1}} \bm{x}'_0 + (1-\sqrt{\bar{\alpha}_{t-1}}) \bm{\delta}, (1-\bar{\alpha}_{t-1}) \bm{I})$, then the statement can be proved by induction.

    We have
    \begin{align}
        q_\sigma (\bm{x}'_{t-1} | \bm{x}'_0) = \int_{\bm{x'_t}} q_\sigma (\bm{x}'_t | \bm{x}'_0) q_\sigma (\bm{x}'_{t-1}|\bm{x}'_t, \bm{x}'_0) d\bm{x}'_t, 
    \end{align}
    where 
    \begin{align*}
    q_\sigma(\bm{x}'_t|\bm{x}'_0) = &\mathcal{N}(\sqrt{\bar{\alpha}_t} \bm{x}'_0 + (1-\sqrt{\bar{\alpha}_t}) \bm{\delta}, (1-\bar{\alpha}_t) \bm{I}), \\
        q_\sigma (\bm{x}'_{t-1}|\bm{x}'_t, \bm{x}'_0) = &\mathcal{N}(\sqrt{\bar{\alpha}_{t-1}} \bm{x}'_0 + (1-\sqrt{\bar{\alpha}_{t-1}}) \bm{\delta}  \\
    &+ \sqrt{1-\bar{\alpha}_{t-1}-\sigma_t^2} \cdot \frac{\bm{x}'_t - \sqrt{\bar{\alpha}_t} \bm{x}'_0 - (1-\sqrt{\bar{\alpha}_t}) \bm{\delta}}{\sqrt{1-\bar{\alpha}_t}}, \sigma_t^2 \bm{I}).
    \end{align*}

    Then from~\cite{bishop2006pattern} (2.115), we can write $q_\sigma (\bm{x}'_{t-1} | \bm{x}'_0)$ as Gaussian distribution, where the mean
    \begin{align}
        \bm{\mu} = &\sqrt{\bar{\alpha}_{t-1}} \bm{x}'_0 + (1-\sqrt{\bar{\alpha}_{t-1}}) \bm{\delta} \\ &+ \sqrt{1-\bar{\alpha}_{t-1}-\sigma_t^2} \cdot \frac{\sqrt{\bar{\alpha}_t} \bm{x}'_0 + (1-\sqrt{\bar{\alpha}_t}) \bm{\delta} - \sqrt{\bar{\alpha}_t} \bm{x}'_0 - (1-\sqrt{\bar{\alpha}_t}) \bm{\delta}}{\sqrt{1-\bar{\alpha}_t}}  \notag \\
        = &\sqrt{\bar{\alpha}_{t-1}} \bm{x}'_0 + (1-\sqrt{\bar{\alpha}_{t-1}}) \bm{\delta},
    \end{align}

    variance
    \begin{align}
        \bm{\Sigma} = \left( \frac{1-\bar{\alpha}_{t-1}-\sigma_t^2}{1-\bar{\alpha}_t} \cdot (1-\bar{\alpha}_t)\right) \bm{I} + \sigma^2_t \bm{I} = (1-\bar{\alpha}_{t-1}) \bm{I}.
    \end{align}

    Hence
    \begin{align}
        q_\sigma(\bm{x}'_{t-1}|\bm{x}'_0) = \mathcal{N}(\sqrt{\bar{\alpha}_{t-1}} \bm{x}'_0 + (1-\sqrt{\bar{\alpha}_{t-1}}) \bm{\delta}, (1-\bar{\alpha}_{t-1}) \bm{I}),
    \end{align}
    which finishes the proof.
\end{proof}

Since we consider DDIM sampling, we can set $\sigma_t=0$ to simplify the derivation. On the other hand, from $q_\sigma(\bm{x}'_t|\bm{x}'_0) = \mathcal{N}(\bm{x}'_t; \sqrt{\bar{\alpha}_t} \bm{x}'_0 + (1-\sqrt{\bar{\alpha}_t}) \bm{\delta}, (1-\bar{\alpha}_t) \bm{I})$, we have
\begin{align}
\label{eq:reverse}
    \bm{x}'_0 = \frac{\bm{x}'_t - \sqrt{1-\bar{\alpha}_t} \bm\epsilon - (1-\sqrt{\bar{\alpha}_t})\bm\delta}{\sqrt{\bar{\alpha}_t}}, \bm\epsilon \sim \mathcal{N}(\bm{0}, \bm{I}).
\end{align}

Given the above reverse transition $q_\sigma (\bm{x}'_{t-1}|\bm{x}'_t, \bm{x}'_0) = \mathcal{N}(\sqrt{\bar{\alpha}_{t-1}} \bm{x}'_0 + (1-\sqrt{\bar{\alpha}_{t-1}}) \bm{\delta} + \sqrt{1-\bar{\alpha}_{t-1}-\sigma_t^2} \cdot \frac{\bm{x}'_t - \sqrt{\bar{\alpha}_t} \bm{x}'_0 - (1-\sqrt{\bar{\alpha}_t}) \bm{\delta}}{\sqrt{1-\bar{\alpha}_t}}, \sigma_t^2 \bm{I})$, substitute $\bm{x}'_0$ with Equation~\ref{eq:reverse}. After rearranging the common terms, the reverse transition can be rewritten as
\begin{align}
\label{eq:transition}
    &q_\sigma (\bm{x}'_{t-1}|\bm{x}'_t, \bm{x}'_0) = \notag \\
    &\frac{\sqrt{\bar\alpha_{t-1}}}{\sqrt{\bar{\alpha}_t}} \left[\bm{x}'_t - \frac{\sqrt{\bar\alpha_{t-1}} \sqrt{1-\bar\alpha_t} - \sqrt{\bar\alpha_t} \sqrt{1-\bar\alpha_{t-1}}}{\sqrt{\bar\alpha_{t-1}}} \left(\bm{\epsilon} + \frac{\sqrt{\bar{\alpha}_{t-1}} - \sqrt{\bar{\alpha}_t}}{\sqrt{\bar{\alpha}_{t-1}} \sqrt{1-\bar{\alpha}_t} - \sqrt{\bar{\alpha}_t} \sqrt{1-\bar{\alpha}_{t-1}}} \bm{\delta}\right)\right].
\end{align}

Equation~\ref{eq:transition} indicates that now we need to train the network to predict $\bm{\epsilon} + \frac{\sqrt{\bar{\alpha}_{t-1}} - \sqrt{\bar{\alpha}_t}}{\sqrt{\bar{\alpha}_{t-1}} \sqrt{1-\bar{\alpha}_t} - \sqrt{\bar{\alpha}_t} \sqrt{1-\bar{\alpha}_{t-1}}} \bm{\delta}$ instead of only $\bm\epsilon$ for backdoor training, which leads to the loss function in Equation~\ref{eq:loss}.

\section{Discussion on the importance and motivation of backdooring diffusion models with invisible triggers}
\label{sec:motivation}

As also discussed in previous work~\cite{chou2023backdoor, chou2023villandiffusion, chen2023trojdiff}, backdooring diffusion models is an important topic for safe utilization of diffusion models. Since the powerful models like Stable Diffusion~\cite{rombach2022high} is open-sourced, anyone could download the model and conduct malicious fine-tuning to insert a secret backdoor that can exhibit a designated action (e.g. generating a inappropriate or incorrect images). Explicitly, the generated output will be directly controlled by activating backdoor for conducting some bad actions like disseminating propaganda, generating fake contents etc. 
Meanwhile, implicitly, as also discussed in~\cite{chou2023backdoor}, the diffusion model has been widely used in a lot of different downstream tasks and applications such as reinforcement learning, object detection, and semantic segmentation~\cite{baranchuk2021label, chen2022offline, chen2023diffusiondet}. Hence if the diffusion model is backdoored, this Trojan effect can bring immeasurable cartographic damage to all downstream tasks and applications.

Given the importance of backdooring diffusion models, exploring invisibility of image triggers could further help the community understand the potential security threat better. As both mentioned in~\cite{doan2021lira, doan2021backdoor}, it is important to improve the fidelity of poisoned examples that are used to inject the backdoor and hence reduce the perceptual detectability by human observers. In the unconditional case, it is thus important to make the sampled noise to be similar with random noise used in the practice or it could be easily filtered by human inspection. As shown in Figure~\ref{fig:trigger} and Figure~\ref{fig:visualize_universal}, the triggers used by previous works (also in the unconditional case) could be easily detected through human inspection without any effort. In contrast, our proposed invisible trigger is nearly visually indistinguishable from the original input, which greatly increase attack's stealth so that human inspection would no longer effective. In addition to unconditional generation, invisible triggers are particularly practical in conditional diffusion models, which hasn't been explored and discussed by the previous works. To be noted, as we show in Section~\ref{sec:watermarking}, the proposed invisible triggers in conditional generation can be directly applied to model watermarking for model ownership verification in practice, further enhancing the significance of our proposed framework.

\section{Preliminary on diffusion models}
\label{sec:diffusion}  
Diffusion models consist of two processes, the forward/diffusion process as a Markov chain and the backward/reverse process~\cite{ho2020denoising}. In the diffusion process, given an image sampled from real data distribution, 
Gaussian noise is gradually added to real data $\bm{x}_0 \in \mathbb{R}^d$ sampled from the real data distribution $q(\bm{x}_0)$ for $T$ steps, producing a series of noisy copies $\bm{x}_1, \bm{x}_2, \cdots, \bm{x}_T \in \mathbb{R}^d$. 
As $T \to \infty$, $\bm{x}_T$ will follow the isotropic Gaussian distribution, i.e., $\bm{x}_T \sim \mathcal{N}(\bm{0}, \bm{I})$. More formally, the diffusion process is defined as
\begin{equation}
\begin{aligned}
    &q(\bm{x}_{1:T}|\bm{x}_0) := \prod_{t=1}^T q(\bm{x}_t|\bm{x}_{t-1}), \quad \quad \\&q(\bm{x}_t|\bm{x}_{t-1}) := \mathcal{N}(\bm{x}_t; \sqrt{1-\beta_t} \bm{x}_{t-1}, \beta_t \bm{I}).
\end{aligned}    
\end{equation}

By defining $\alpha_t := 1 - \beta_t$, $\bar{\alpha}_t := \prod_{s=1}^t \alpha_s$, we have
\begin{align}
    q(\bm{x}_t|\bm{x}_0) = \mathcal{N}(\bm{x}_t; \sqrt{\bar{\alpha}_t} \bm{x}_0, (1-\bar{\alpha}_t) \bm{I}).
\end{align}

We can simulate the true data distribution by reversing the diffusion process as described above. Hence the reverse process can also be defined as a Markov chain with learned Gaussian transitions starting from $p(\bm{x}_T)=\mathcal{N}(\bm{x}_T; \bm{0}, \bm{I})$:
\begin{equation}
\begin{aligned}
    &p_\theta (\bm{x}_{0:T}) := p(\bm{x}_T) \prod_{t=1}^T p_\theta (\bm{x}_{t-1}|\bm{x}_t), \quad \quad\\ 
    &p_\theta (\bm{x}_{t-1}|\bm{x}_t) := \mathcal{N} (\bm{x}_{t-1}; \bm{\mu}_\theta (\bm{x}_t, t), \bm{\Sigma}_\theta (\bm{x}_t, t)).
\end{aligned}    
\end{equation}

The training is performed by optimizing the variational lower bound, which can be further rewritten as comparing the KL divergence between the $p_\theta (\bm{x}_{t-1}|\bm{x}_t)$ and $q(\bm{x}_{t-1}|\bm{x}_t, \bm{x}_0)$ as
\begin{equation}
\begin{aligned}
    &q(\bm{x}_{t-1}|\bm{x}_t, \bm{x}_0) = \mathcal{N}(\bm{x}_{t-1}; \tilde{\bm{\mu}}_t(\bm{x}_t, \bm{x}_0), \tilde{\beta}_t \bm{I}), \quad \quad \\
    &\tilde{\bm{\mu}}_t(\bm{x}_t, \bm{x}_0) = \frac{1}{\sqrt{\alpha_t}} \left(\bm{x}_t (\bm{x}_0, \bm{\epsilon}) - \frac{\beta_t}{\sqrt{1 - \bar{\alpha}_t}} \bm{\epsilon} \right),
\end{aligned}    
\end{equation}
where $\bm{x}_t(\bm{x}_0, \bm{\epsilon}) = \sqrt{\bar{\alpha}_t} \bm{x}_0 + \sqrt{1 - \bar{\alpha}_t} \bm{\epsilon}, \bm{\epsilon} \sim \mathcal{N}(\bm{0}, \bm{I})$.

Due to the property of Gaussian distribution, 
the loss function can be further written as
\begin{equation}
\begin{aligned}
    L = \mathbb{E}_{t, \bm{x}_0, \bm{\epsilon}} \left[\|\bm{\epsilon} - \bm{\epsilon}_\theta(\sqrt{\bar{\alpha}_t} \bm{x}_0 + \sqrt{1 - \bar{\alpha}_t} \bm{\epsilon}, t)\|^2\right].
\end{aligned}    
\end{equation}

\section{Discussion on learnable invisible triggers through bi-level optimization in classification models}
\label{sec:invisible_cls}

As mentioned in Section~\ref{sec:intro}, learning invisible triggers by bi-level optimization in diffusion models is different and much harder compared to finding one in classification models. The method developed for backdooring classification models cannot be directly or easily extended to backdoor diffusion models. Specifically, the threat model is totally different. Diffusion models consist of diffusion and reverse processes that fundamentally differs from classification models. Backdooring diffusion model needs to have careful control of the training procedure while only poisoning data needs to be added in the classification model. At the same time, it is nontrivial and challenging to design the backdoor objective in the conditional and unconditional diffusion model  while it is relatively a simple task in the classification. To learn invisible backdoors for both unconditional and conditional diffusion models, the entire pipeline, training paradigm, and training loss have to be redesigned  to differ significantly when applying bi-level optimization to backdoor diffusion models.
In this setting, the training loss, training paradigm, and pipeline are specifically designed based on the properties of diffusion models differing substantially from backdooring classification models through bi-level optimization.

\section{Training procedure for backdooring unconditional diffusion models}
\label{sec:training_unconditional}

The training algorithm for backdooring unconditional diffusion models is shown in Algorithm~\ref{alg:unconditional}.

\begin{algorithm}[t]
	\caption{Backdoored diffusion model training given the prior is random noise, i.e., unconditional generation.}  
	\label{alg:unconditional}
	\begin{algorithmic}[1]
		\STATE {\bfseries Input:} $K,D$, stepsizes $\alpha$ and $\beta$, initializations $\bm\delta_0$ and $\theta_0$, target image $\bm{y}$, dataset $\mathcal{D}=\{\mathcal{D}_c, \mathcal{D}_p\}$. 
		\FOR{$k=0,1,2,...,K$}
		\STATE{Set $\bm\delta_k^0 = \bm\delta_{k-1}^{D} \mbox{ if }\; k> 0$ and $\bm\delta_0$ otherwise 
		}
		\FOR{$i=1,....,D$}
		\STATE{$\tilde{\bm{\epsilon}} \sim \mathcal{N}(\bm{0}, \bm{I})$}  
		\STATE{$\tilde{\bm\delta}_k^i = \bm\delta_k^{i-1}-\alpha \nabla_{\bm\delta} L_{inner}(\bm{\epsilon}_{\theta_k},\tilde{\bm{\epsilon}}, \bm{\delta}_k^{i-1})$}
            \STATE{$\bm\delta_k^i=\mathrm{Proj}_{\|\cdot\|_\infty \leq C}(\tilde{\bm\delta}_k^i$)}
		\ENDFOR
                 \STATE{$\bm{x}_0 \sim \{\mathcal{D}_c, \mathcal{D}_p$\}}
                 \STATE{$t \sim \text{Uniform}(\{1, \ldots, T\})$}
                 \STATE{$\bm{\epsilon} \sim \mathcal{N}(\bm{0}, \bm{I})$}

                  \STATE{Compute gradient $\nabla_\theta L_{outer}\left(\bm{\epsilon}_{\theta_k}, \bm{x}_0, \bm{\epsilon}, \bm{\delta}_k^D, \bm{y}, t\right)$}
                 \STATE{Let $\theta_{k+1}=\theta_k- \beta \nabla_\theta L_{outer}\left(\bm{\epsilon}_{\theta_k}, \bm{x}_0, \bm{\epsilon}, \bm{\delta}_k^D, \bm{y}, t\right)$}
		\ENDFOR
	\end{algorithmic}
	\end{algorithm}	

\section{Multiple universal trigger-target pairs}
\label{sec:uncondition_multiple}

To further show the capability of the proposed framework, we show it is possible to learn multiple universal trigger-target pairs simultaneously. The results with two trigger-target pairs on CIFAR10 are shown in Table~\ref{tab:multiple}, which indicate that the framework can be directly extended to learn multiple universal trigger-target pairs.
\begin{table}[htbp]
\centering
\caption{Results on two universal trigger-target pairs, which show the attack is still very effective with two trigger-target pairs.}
\begin{tabular}{c|c|c|c}
\hline
& FID & MSE for first target & MSE for second target \\ \hline
Clean model &  12.80 & - & - \\
Backdoored model & 13.77 & 4.40e-3 & 2.33e-6 \\ \hline
\end{tabular}
\label{tab:multiple}
\end{table}


\section{Experiments on different samplers}
\label{sec:samplers}
We also conducted experiments on different samplers, DPMSolver~\cite{lu2022dpm} to show that the proposed loss in unconditional generation can be directly applied to different commonly used samplers. Previous work~\cite{chou2023backdoor} only consider DDPM sampling and the trained backdoored diffusion models cannot be used for other samplers. Figure~\ref{fig:ddim_prev} shows the sampling results with previous work's backdoor models where the left one is triggered inputs and the right one is sampling results, indicating the backdoor is ineffective when other samplers are used. In our proposed framework, however, different commonly used samplers can be used. We use second-order DPMSolver to test the backdoor performance. As shown in Figure~\ref{fig:visualize_dpm} and Table~\ref{tab:visualize_dpm}, the injected backdoor is still very effective.

\begin{table}[ht]
\centering
\caption{FID and MSE results when applying DPMSolver sampler, which indicate the backdoor is still effective under different samplers.}
\begin{tabular}{c|c|c}
\hline
& FID & MSE \\
\hline
Clean model & 12.80 & - \\
Backdoored model & 9.50 & 3.10e-3 \\
\hline
\end{tabular}
\label{tab:visualize_dpm}
\end{table}

\begin{figure}[h]
  \centering
  \begin{minipage}{0.3\textwidth}
    \centering
    \includegraphics[width=\linewidth]{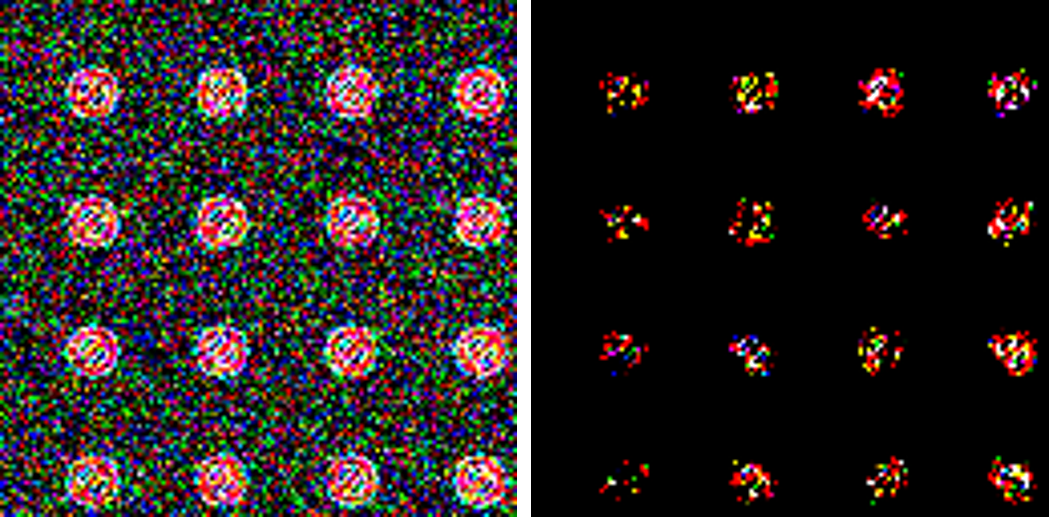} 
    \caption{Illustration of previous backdoor~\cite{chou2023backdoor} for DDIM sampler on CIFAR10. The left figure is the initial noise with the visible triggered and the right figure is the generated output from sampling.}
    \label{fig:ddim_prev}
  \end{minipage}
  \hfill
  \begin{minipage}{0.3\textwidth}
    \centering
    \includegraphics[width=\linewidth]{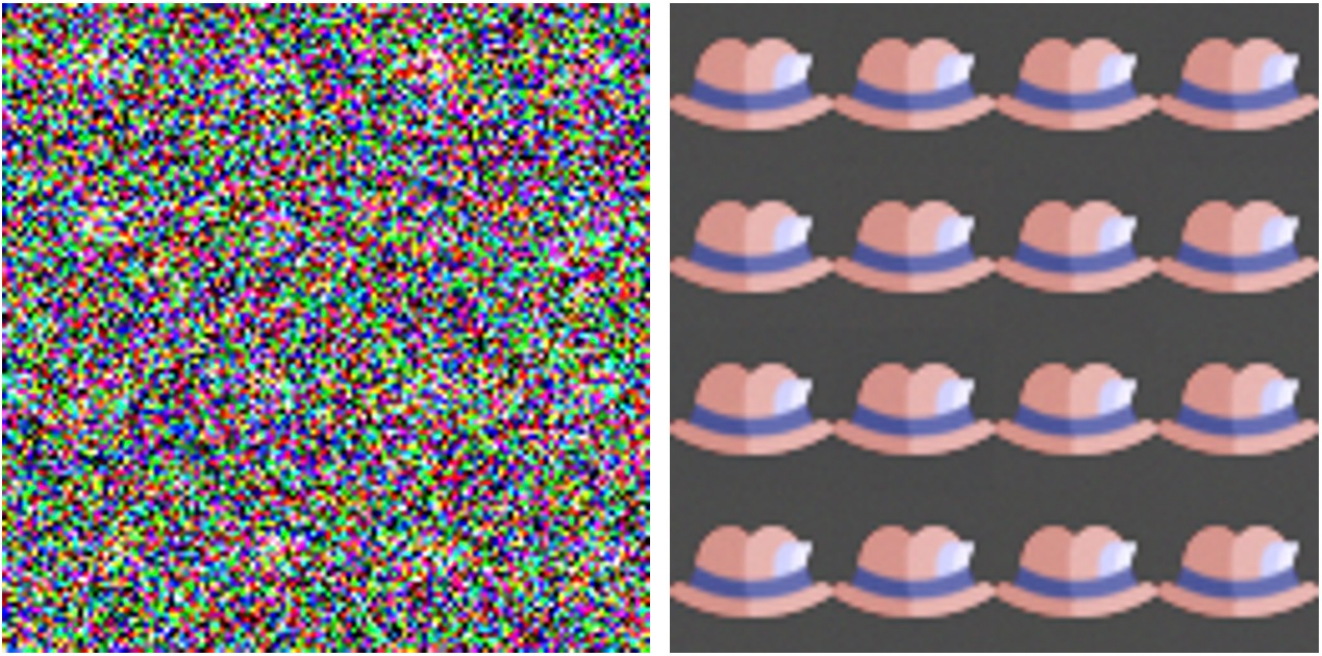} 
    \caption{Visualization results of DPMSolver sampler on CIFAR10.}
    \label{fig:visualize_dpm}
  \end{minipage}
  \hfill
  \begin{minipage}{0.3\textwidth}
    \centering
    \includegraphics[width=\linewidth]{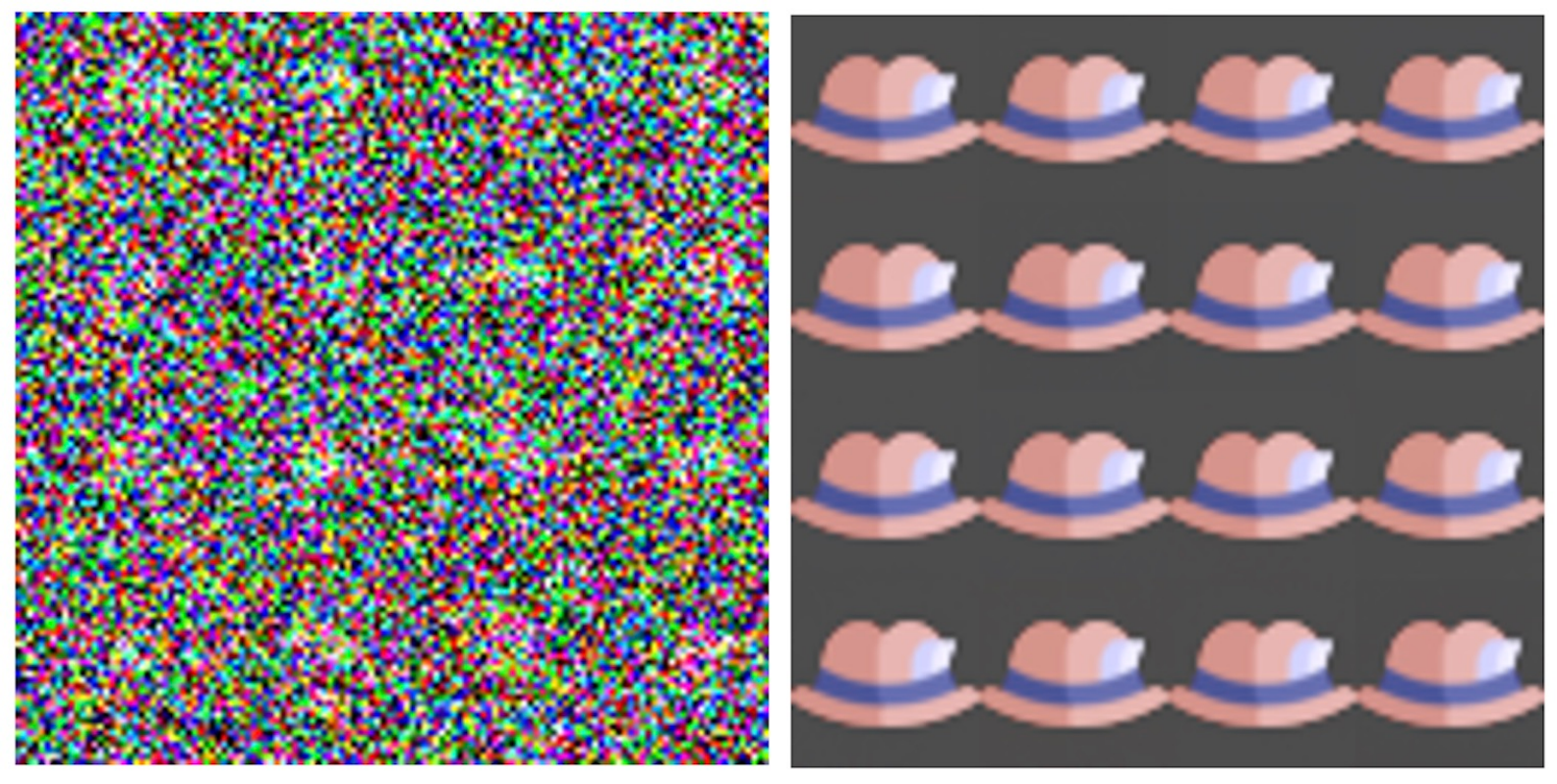} 
    \vspace{-1em}
    \caption{Visualization results w/ clip operation for backdoor sampling.}
    \label{fig:visualize_clip}
  \end{minipage}
\end{figure}


\section{Multiple input-aware trigger-target pairs}
\label{sec:multiple_condition}
Recall that for the conditional case, the inputs to the trigger generator are masked image, mask, and target image. Hence if we use different target images, can we insert multiple targets simultaneously? Here we show it is possible to insert multiple targets during training. Specifically, we use two target images (`Hat' and `Shoe') to train the trigger generator to learn input-aware invisible triggers based on masked image and target image. The results are shown in Table~\ref{tab:condition_multiple}. 

\begin{table}[htbp]
\centering
\caption{Results on inserting two target images simultaneously, indicating the possibility of a even stronger attack with different targets.}
\begin{tabular}{c|c|c|c|c}
\hline
& FID & LPIPS & MSE for first target & MSE for second target \\ \hline
Clean model &  1.00 & 0.064 & - & - \\
Backdoored model & 1.02 & 0.063 & 1.44e-3 & 2.07e-3 \\ \hline
\end{tabular}
\label{tab:condition_multiple}
\vspace{-1.5em}
\end{table}
\section{Defense against backdoored diffusion models}
\label{sec:clip}
In this section, we test various defense methods against the proposed framwork. As a baseline method, we first test if finetuning with clean data can mitigate the backdoors. Specifically, we finetune a backdoored model (trained on CIFAR10 with norm bound 0.1) with all clean training data of CIFAR10 for five epochs. Then we sample 10K backdoored images to compute the MSE with the target image. The average MSE is 7.03e-3, similar to the results in Table~\ref{tab:visualize_universal}, indicating that the backdoor is still effective and even finetuning with all clean training data cannot remove the inserted backdoor.


Then we show that both ANP~\cite{wu2021adversarial} and inference-time clipping mentioned in previous work~\cite{chou2023backdoor} become totally ineffective in our proposed framework. We first present defense results on ANP against a backdoored diffusion model trained on CIFAR10 with norm bound 0.2 and poison rate 0.1. Following the settings in~\cite{chou2023backdoor}, we use the largest perturbation budget (budget=4, larger budge means better Trojan detection) in~\cite{chou2023backdoor} and train the perturbated model with the whole clean dataset for 5 epochs. With different learning rates($1e-4, 2e-4$), we found ANP performs even worse on our proposed attack, compared to the performance on the attack in~\cite{chou2023backdoor}. The perturbated model immediately collapses to a meaningless image or a black image. The visualization results with different learning rates during the training are shown in Figure~\ref{fig:anp1} and Figure~\ref{fig:anp2}. This can also observed from the MSE results between the reversed target image by ANP and the true target image (`Hat' in our experiments), as shown in Table~\ref{tab:anp1}. We sample 2048 images to compute the MSE, same as~\cite{chou2023backdoor}. As shown in the tables, the computed MSE values are large, indicating that ANP cannot reconstruct the target image at all. We then demonstrate the defense results of inference-time clipping. With the clip operation in~\cite{chou2023backdoor}, we sample images with DDIM~\cite{song2020denoising} sampling with different poison rates. As shown in Figure~\ref{fig:visualize_clip} and Table~\ref{tab:visualize_clip}, with clipping, backdoored models can still achieve high-utility and high-specificity, which indicates the defense method is not a good choice for these cases.

\begin{wrapfigure}{r}{0.45\textwidth}
  \begin{center}
    \includegraphics[width=0.4\textwidth]{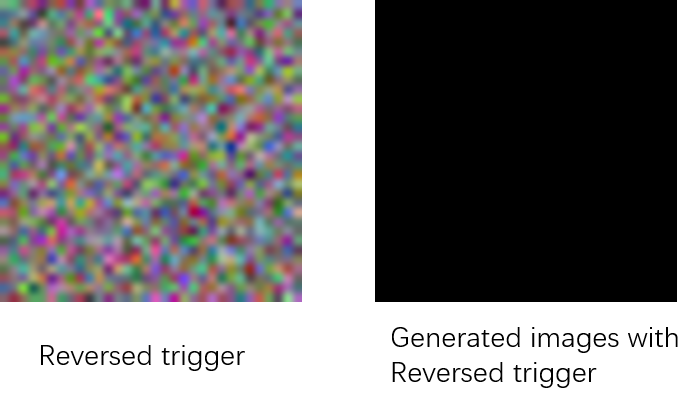}
  \end{center}
  \vspace{-1em}
  \caption{Illustration for the reversed trigger by Elijah~\cite{an2024elijah} and the generated image given input with the reversed trigger, indicating Elijah~\cite{an2024elijah} is not effective in our proposed framework.}
  \label{fig:reverse}
  \vspace{-1em}
\end{wrapfigure}
Moreover, we also test the recently proposed defense method Elijah~\cite{an2024elijah} which is specifically designed for backdoors in diffusion models. We test Elijah in the unconditional case, on a backdoored model trained on CIFAR10 with norm bound 0.2. As a reverse-engineering based method, if Elijah cannot reverse the trigger effectively and generate the similar target image with the reversed trigger, then it cannot defend the backdoored model effectively. Our experiments show that Elijah can only reverse an incorrect and noisy trigger on the backdoored model. More importantly, when the reversed trigger is added onto the input, the model totally collapses and can only generate black images whereas the true target image is the ‘Hat’ image, which means that Elijah totally fails on the attacks. An illustration for the reversed trigger and generated image is shown in Figure~\ref{fig:reverse}. Furthermore, please note Elijah cannot be applied to our proposed attack for the conditional case. The reason is that Elijah reverses and detects the backdoors by the distribution shift in the input noise. However, in our proposed attack for the conditional case, there is no such distribution shift in the input noise of clean models and backdoored models. In this case, the optimization loss for the reverse of the trigger will directly become 0. Hence Elijah naturally cannot be used to defend our proposed attack for the conditional generation.

To summarize, as shown above, our proposed framework is robust to various strong backdoor mitigation methods, demonstrating the stealthiness and effectiveness.

\begin{figure}[ht]
  \centering
  \begin{minipage}{0.45\textwidth}
    \centering
    \includegraphics[width=\linewidth]{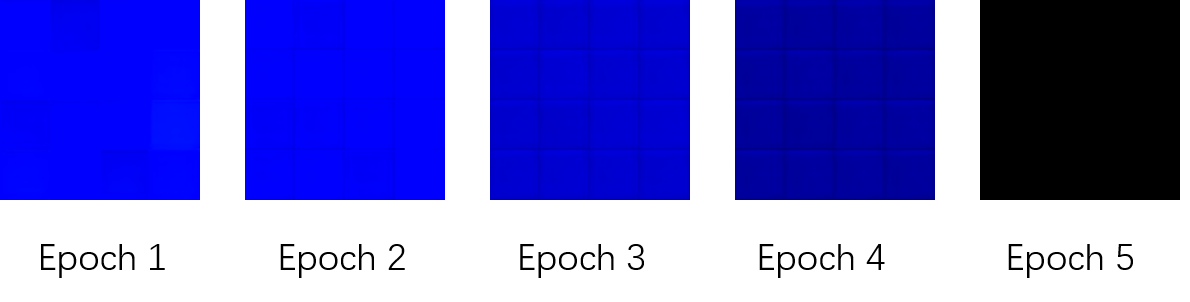} 
    \vspace{-1em}
    \caption{Reversed target images by ANP with learning rate $1e-4$.}
    \label{fig:anp1}
  \end{minipage}
  \hfill
  \begin{minipage}{0.45\textwidth}
    \centering
    \includegraphics[width=\linewidth]{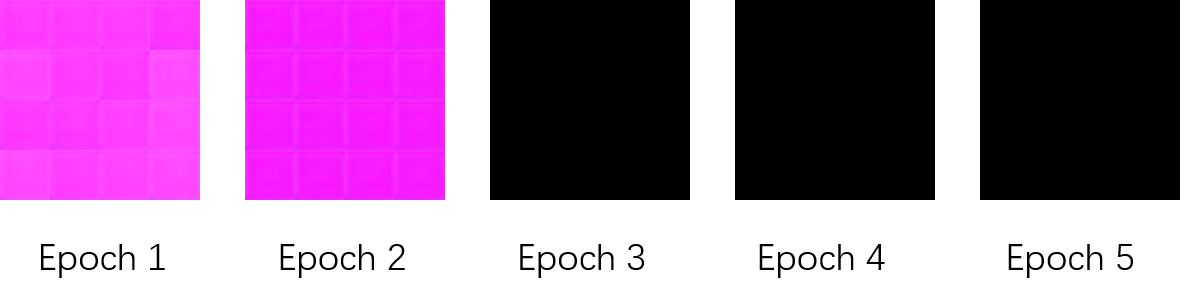} 
    \vspace{-1em}
    \caption{Reversed target images by ANP with learning rate $2e-4$.}
    \label{fig:anp2}
  \end{minipage}
\end{figure}
\vspace{-1em}
\begin{table}[ht]
\centering
\caption{MSE between reversed target images and the true target image. Learning rate: $1e-4, 2e-4$.}
\begin{tabular}{c|c|c|c|c|c}
\hline
LR & Epoch 1 & Epoch 2 & Epoch 3 & Epoch 4 & Epoch 5 \\
\hline
$1e-4$ & 0.28 & 0.28 & 0.22 & 0.19 & 0.24 \\
\hline
$2e-4$ & 0.24 & 0.26 & 0.24 & 0.24 & 0.24 \\
\hline
\end{tabular}
\label{tab:anp1}
\end{table}



\begin{table}[H]
    \begin{minipage}{0.45\textwidth}
    \centering    
      \caption{Quantitative results w/ and w/o clip operation on CIFAR10.}
        \scalebox{0.84}{\begin{tabular}{c|c|c|c|c}
\hline
& \multicolumn{2}{c|}{w/o clip} & \multicolumn{2}{c}{w/ clip} \\
\hline
Poison rate & FID & MSE & FID & MSE \\
\hline
0.05 & 11.76 & 3.07e-3 & 11.76 & 3.49e-3 \\ \hline
0.3 & 14.98 & 6.36e-5 & 14.98 & 8.59e-5 \\ \hline
0.4 & 15.88 & 2.29e-6 & 15.88 & 2.29e-6 \\ \hline
0.5 & 18.31 & 2.29e-6 & 18.33 & 2.29e-6 \\ 
\hline
\end{tabular}}
\label{tab:visualize_clip}
    \end{minipage}\hfill%
    \begin{minipage}{0.45\textwidth}
      \centering
        \caption{Results on finetuning pre-trained models with different poison rates.}
        \scalebox{0.76}{\begin{tabular}{c|c|c|c|c}
\hline
& \multicolumn{2}{c|}{Poison rate=0.1} & \multicolumn{2}{c}{Poison rate=0.5} \\ \hline
Finetuning epochs & FID & MSE & FID & MSE \\ \hline
30 & 8.22 & 3e-5 & 8.55 & 3.14e-6 \\ \hline
100 & 6.40 & 6.12e-6 & 6.20 & 2.34e-6 \\ \hline
    \end{tabular}}
\label{tab:finetune}
    \end{minipage} 
\end{table}

\section{Results on finetuning pretrained models}
\label{sec:finetune}
Here, we showcase the effectiveness of our proposed framework by fine-tuning pre-trained models for varying numbers of epochs. The results are presented in Table~\ref{tab:finetune}. It is worth noting that we can successfully introduce a backdoor into the model by fine-tuning it for as few as 30 epochs, yet still achieve a lower FID compared to the clean model. This means the proposed framework can easily be applied in practice.

\end{document}